\documentclass[12pt]{article}
\usepackage[letterpaper, margin=1in]{geometry}

\usepackage{amsmath,amsfonts,graphicx,amsthm}
\usepackage{amssymb,tikz}
\usepackage{enumitem}
\usepackage{bbm}
\usepackage{xcolor}
\usepackage{algorithm}
\usepackage{algpseudocode}
\usepackage{float}
\usepackage{subcaption}
\usepackage[round, authoryear]{natbib}
\usepackage{hyperref}
\hypersetup{hidelinks,colorlinks=true,linkcolor=blue,citecolor=blue,linktocpage=true}
\newtheorem{theorem}{Theorem}[section]

\newtheorem{proposition}{Proposition}
\newtheorem{lemma}{Lemma}

\DeclareMathOperator*{\argmin}{argmin}

\allowdisplaybreaks

\title{Time-uniform conformal and PAC prediction\footnote{Some content of this work has been presented by the second author at IMS Annual Meeting 2022, London.}}
\author{Kayla E. Scharfstein$^1$ and Arun Kumar Kuchibhotla$^1$}
\date{$^1$Department of Statistics and Data Science, Carnegie Mellon University\\[2ex]%
\today}

\begin{document}

\maketitle

\begin{abstract}
Given that machine learning algorithms are increasingly being deployed to aid in high stakes decision-making, uncertainty quantification methods that wrap around these black box models such as conformal prediction have received much attention in recent years. In sequential settings, where data are observed/generated in a streaming fashion, traditional conformal methods do not provide any guarantee without fixing the sample size. More importantly, traditional conformal methods cannot cope with sequentially updated predictions. As such, we develop an extension of the conformal prediction and related probably approximately correct (PAC) prediction frameworks to sequential settings where the number of data points is not fixed in advance. The resulting prediction sets are anytime-valid in that their expected coverage is at the required level at any time chosen by the analyst even if this choice depends on the data. We present theoretical guarantees for our proposed methods and demonstrate their validity and utility on simulated and real datasets.
\end{abstract}

\section{Introduction} \label{introduction_sec}

As black box machine learning algorithms become increasingly prevalent in high stakes decision-making processes, it is crucial that methods for quantifying the uncertainty of their predictions keep pace. For example, complex algorithms are used to inform decisions about bail, sentencing, and parole in the criminal justice system, support donor-recipient matching in organ transplantation, and detect nearby objects in self-driving cars \citep{angwin2016, briceno2020, badue2021}. In these contexts, quantifying the uncertainty of predictions can help decision-makers hedge against costly wrong choices and calibrate their expectations for the true outcome. 

Because the distribution of the data used in the decision-making pipeline is typically unknown, a popular line of research for uncertainty quantification has been to construct distribution-free prediction sets that wrap around the black box machine learning model. As described in \cite{vovk2005}, conformal prediction is one popular framework that provides finite-sample coverage guarantees. Given a desired nominal miscoverage level $\alpha \in (0, 1)$, the goal of conformal prediction is to construct a prediction set $\widehat{C}_{n, \alpha}$ using IID observations $\{Z_i\}_{i=1}^n$ satisfying \begin{equation} \label{eq:conformal_goal}
    \mathbb{P}(Z \in \widehat{C}_{n, \alpha}) ~\geq~ 1 - \alpha
\end{equation} 
for a test point $Z$ that is independent of $\{Z_i\}_{i=1}^n$ and shares the same distribution. Here, the probability is taken over the joint distribution of $\{Z_i\}_{i=1}^n$ and $Z$. Condition~\eqref{eq:conformal_goal} can be equivalently phrased in terms of the probability measure of $Z$: If $\mu_Z(A) = \mathbb{P}(Z \in A)$ for any Borel set $A$, then inequality~\eqref{eq:conformal_goal} is the same as
\begin{equation}\label{eq:equivalent-conformal-goal}
    \mathbb{E}[\mu_Z(\widehat{C}_{n,\alpha})] ~\ge~ 1 - \alpha.
\end{equation}
Note that $\mu_Z(\widehat{C}_{n,\alpha})\in[0, 1]$ is a random variable because $\widehat{C}_{n,\alpha}$ is a random set constructed using $\{Z_i\}_{i=1}^n$.
Several methods have been proposed that achieve \eqref{eq:conformal_goal} by relying only on the exchangeability of $(\{Z_i\}_{i=1}^n,\, Z)$ including the full conformal method \citep{vovk2005, lei2013} and the split conformal method \citep{papadopoulos2002, lei2014}.

A related line of research inspired by classical tolerance regions \citep{guttman1970, krishnamoorthy2009} is concerned with constructing the so-called probably approximately correct (PAC) prediction sets. The goal differs from that of the conformal approach in that the \textit{training data conditional} coverage of the prediction set should be approximately correct with high probability: \begin{equation} \label{eq:pac_goal}
    \mathbb{P}(\mathbb{P}(Z \in \widehat{C}_{n, \alpha, \delta} \vert Z_1, \dots, Z_n) \geq 1 - \alpha) \geq 1 - \delta\quad\equiv\quad \mathbb{P}(\mu_Z(\widehat{C}_{n,\alpha}) \ge 1-\alpha) \ge 1 - \delta.
\end{equation} The prediction set $\widehat{C}_{n, \alpha, \delta}$ is said to be $(\alpha, \delta)$-PAC if it satisfies \eqref{eq:pac_goal}. Note that an additional parameter $\delta \in (0, 1)$ is required to formulate the PAC guarantee which is absent from the conformal framework. \cite{vovk2005} shows that the inductive conformal method can also be used to construct a PAC prediction set. Other methods assume the data $\{Z_i\}_{i=1}^n$ and the test point $Z$ to be independent and identically distributed (IID) so that concentration results such as the Dvoretzky-Kiefer-Wolfowitz inequality \citep{massart1990} can be applied to achieve the PAC guarantee; see, for example, \cite{gyorfi2019} and \cite{yang2021}.

Notably, in both \eqref{eq:conformal_goal} and \eqref{eq:pac_goal}, the probabilities are taken over the observed data $\{Z_i\}_{i=1}^n$ and the test data point $Z$ while the sample size $n$ is treated as fixed. Thus, an implicit assumption of the previously described conformal and PAC formulations is that the sample size $n$ is known in advance. In settings where data are observed sequentially, this assumption often does not hold in practice. For example, if the data are expensive to obtain, a practitioner might want to stop accruing data as soon as some goal is reached, e.g. the prediction set is precise enough. 

In this paper, we extend the conformal and PAC frameworks to settings where data points are observed IID in a sequential manner and the sample size is a random quantity which might depend on the data. More concretely, consider a scenario where an analyst observes $Z_1$, reports a prediction set $\widehat{C}_{1}$ based on $Z_1$, observes $Z_2$, reports a prediction set $\widehat{C}_{2}$ based on $Z_1$ and $Z_2$, and so on until they decide to stop from some reason at time $T$. We would like to guarantee that the prediction set $\widehat{C}_T$ has appropriate coverage. If we know that the random time $T$ is larger than some $t_0 \geq 1$ with probability 1, then one can always return $\widehat{C}_{t_0, \alpha}$ ($\widehat{C}_{t_0, \alpha, \delta}$) obtained from the fixed sample size conformal (PAC) prediction method and the resulting prediction set trivially has appropriate coverage. Unfortunately, this naive method suffers from sub-optimality in that it is a trivial set for sample sizes smaller than $t_0$, and for $T$ much larger than $t_0$, the width is significantly larger than the best possible. 

Sequential construction of prediction sets is not new and has been considered in the works of~\cite{gibbs2021adaptive,gibbs2024conformal},~\cite{zaffran2022adaptive},~\cite{bhatnagar2023improved}, and~\cite{angelopoulos2024conformal}, among others. Their goal is much more ambitious than ours because the guarantee is long-run average coverage without any assumptions on the underlying data. We, on the other hand, assume IID observations and want coverage at arbitrary stopping times (not long-run coverage). After the initial inception of our method,~\cite{avelin2023} proposed a sequential inductive prediction interval (SIPI) based on a time-uniform version of the Dvoretzky-Kiefer-Wolfowitz inequality from \cite{howard2022}, to obtain appropriate coverage in a PAC sense. Their procedure also provides uniform coverage over the miscoverage levels $\alpha\in[0,1]$ and is thus sub-optimal in terms of width if one is only seeking time-uniform coverage (for a fixed $\alpha$). Additionally, their method has no conformal equivalent (i.e., the analog of~\eqref{eq:conformal_goal}). \cite{gauthier2025} consider a scenario where data arrive sequentially in batches and show how e-values can be used to construct a sequence of batch anytime-valid conformal prediction sets. This differs from our setup in that we target uniform coverage over the observations, while they target uniform coverage over the batches. Practically speaking, their approach is suited to, for example, constructing simultaneously valid prediction sets for the effect of a drug for new patients at hospitals where the drug is sequentially deployed, whereas our approach is suited to constructing simultaneously valid prediction sets for the effect of a drug for patients within a hospital as they are sequentially given the drug.

The notion of time-uniform prediction sets originates from the pragmatic requirement for sequential prediction methods that minimize memory by avoiding the storage of all prior data. In typical applications of prediction sets, the true label is revealed post-prediction, offering a chance to incorporate this new information for enhancement. When considering such iterative improvements under the assumption of IID data, time-uniform validity stands as the uniquely appropriate objective. As illustrated in Figure~\ref{fig:changing_dist}, such updating prediction sets sequentially also has an added benefit of potentially adapting to distribution shifts.

Formally, the time-uniform conformal (TUC) prediction goal is to construct a prediction set $\widehat{C}_{T, \alpha}$ based on the first $T$ observations $\left\{Z_t\right\}^T_{t=1}$ with the guarantee that \begin{equation} \label{eq:tuc_goal}\tag{TUC}
    \mathbb{P}(Z \in \widehat{C}_{T, \alpha}) \geq 1 - \alpha,\quad\mbox{or equivalently,}\quad \mathbb{E}[\mu_Z(\widehat{C}_{T,\alpha})] \ge 1 - \alpha.
\end{equation} This is analogous to \eqref{eq:conformal_goal}, with the fixed sample size $n$ replaced by the random time $T$. Without making any restrictive distributional assumptions about the random time $T$, we only know that it takes on values in $\mathbb{N}^{+}$ and, thus, in order to guarantee that $\widehat{C}_{T, \alpha}$ is a valid TUC prediction set, the best one can do is to enforce that the fixed-time prediction set $\widehat{C}_{t, \alpha}$ has adequate coverage simultaneously at every time $t$ in the full support of the random quantity $T$ (i.e. at each time $t \geq 1$). The following proposition makes clear the connection between the TUC prediction set $\widehat{C}_{T, \alpha}$ and the sequence of fixed-time prediction sets $\{\widehat{C}_{t, \alpha}\}_{t=1}^{\infty}$. \begin{proposition} \label{prop:equiv_TUC_goal}
Suppose $Z_t, Z \sim \mu_Z$ for each $t = 1, 2, \dots$. Then,
\begin{equation}
    \mathbb{P}(Z \in \widehat{C}_{T, \alpha}) \geq 1 - \alpha \text{ for all random times $T$ } \iff \mathbb{E}\left[\min_{t \geq 1} \mu_Z(\widehat{C}_{t, \alpha})\right] \geq 1 - \alpha.
\end{equation}
\end{proposition}
\begin{proof}
A detailed proof is given in Section \ref{equiv_TUC_proof}.
\end{proof} 
Proposition \ref{prop:equiv_TUC_goal} implies that constructing a $(1-\alpha)$-TUC prediction set $\widehat{C}_{T, \alpha}$ is equivalent to constructing an infinite sequence of fixed-time prediction sets $\{\widehat{C}_{t, \alpha}\}_{t=1}^{\infty}$ with expected minimum probability content exceeding $1-\alpha$. Based on this reformulation of the TUC prediction goal, it is clear that it is much stronger than the usual conformal prediction goal \eqref{eq:conformal_goal} for a fixed sample size $n$ when the observations and test point are drawn IID from $\mu_Z$. The time-uniform probably approximately correct (TUPAC) prediction goal can be formulated analogously to the TUC prediction goal. See Proposition \ref{prop:equiv_TUPAC_goal} in Section \ref{proofs} for details.

To further drive home the point that the TUC prediction guarantee is much stronger than its fixed-sample-size counterpart, we illustrate how split conformal prediction sets fail to provide appropriate time-uniform coverage through a toy example. In particular, suppose that an analyst's stream of data comes from a 1-dimensional standard Normal distribution so that we know the true probability content of any prediction set $\widehat{C}_{t, \alpha}$ as well as the oracle parametric prediction set $[z_{\alpha/2}, z_{1-\alpha/2}]$ where $z_{\alpha/2}$ is the $(1-{\alpha}/{2})$-quantile of the Normal distribution. In Figure \ref{fig:numerical_illustration}, we see that the true probability content of prediction sets constructed using the split conformal prediction algorithm concentrates around the desired $1-\alpha$ coverage level. This behavior is desirable when the sample size is fixed in advance and implies that the fixed-sample-size conformal prediction guarantee is satisfied, but it violates the goal of having simultaneous coverage at every time $t \geq 1$. As suggested by Proposition \ref{prop:equiv_TUC_goal}, one would instead hope that the expected minimum probability content over all prediction sets $\widehat{C}_{t, \alpha}$ exceeds the desired coverage level $1-\alpha$. This behavior is exhibited by our proposed split TUC prediction sets constructed using Algorithm \ref{fixed_anytime_alg} in Figure \ref{fig:numerical_illustration}. Likewise, our proposed split TUPAC prediction sets as well as prediction sets built using confidence sequences (CSs) from \cite{howard2022} and the SIPIs proposed in \cite{avelin2023} satisfy an equivalent time-uniform PAC guarantee as in Proposition \ref{prop:equiv_TUPAC_goal}, though the SIPIs are much more conservative because they additionally guarantee uniform coverage over the miscoverage level $\alpha$.

\begin{figure}[h]
    \centering
    \includegraphics[width=0.9\textwidth]{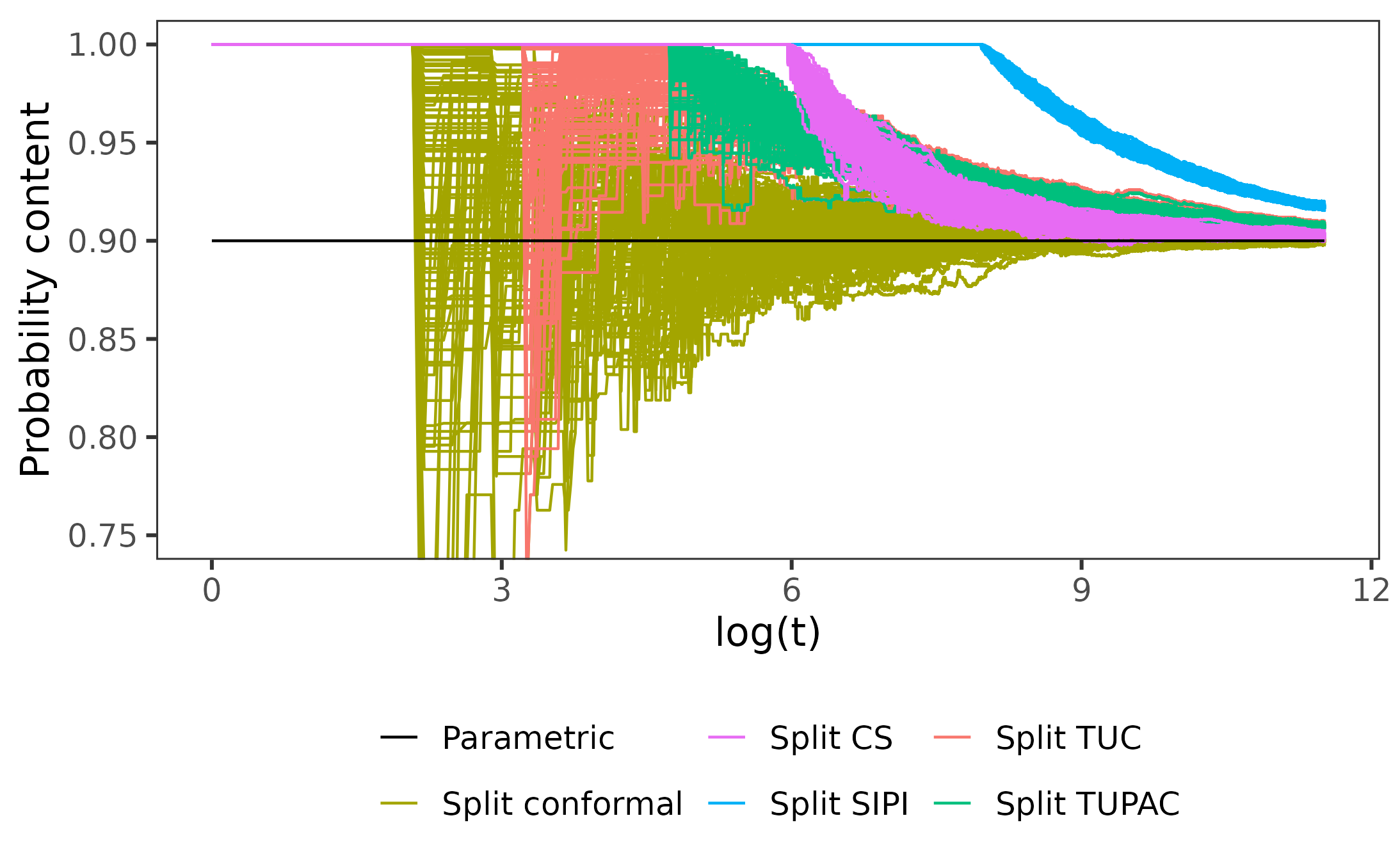}
    \caption{The true probability content of prediction sets $\{\widehat{C}_{t, \alpha, \delta}\}_{t=1}^{100,000}$ constructed using split conformal, TUC, TUPAC, CS, and SIPI algorithms over 100 replications where $\alpha = \delta = 0.1$. The prediction intervals use the transformation $R(Z) = |Z - \bar{Z}|$ (signed for the SIPI algorithm) where $\bar{Z}$ is the empirical mean based on a separate dataset of 100 standard Normal random variables drawn independently from the data stream. To compute the split TUC and TUPAC prediction sets, we take the function $h$ in Algorithm \ref{fixed_anytime_alg} to be the PMF of discretized log Normal random variables with mean 11 and variance 1, respectively (i.e. $h$ is the PMF of $Y = \lfloor X \rfloor$ where $X \sim \text{Lognormal}(11, 1)$). In computing the split CS prediction sets, we choose $u_t$ in Theorem \ref{fixed_confidence_sequence_pac_thm} using the Beta-Binomial mixture approach described in \protect\cite{howard2022}.}
    \label{fig:numerical_illustration}
\end{figure}

The empirical minimum probability content over 100 replications of 100,000 point data streams for split conformal and split TUC prediction intervals across various choices of the target coverage level $1-\alpha$ is compiled in Table \ref{tab:numerical_illustration}. While the empirical minimum probability content of split conformal prediction intervals is consistently below the desired $1-\alpha$ level, the split TUC prediction intervals are better calibrated.

\begin{table}[!h] 
\centering 
\begin{tabular}{lccc} \hline
    & \multicolumn{3}{c}{\textbf{Empirical minimum probability content}} \\
    & $1 - \alpha = 0.9$ & $1- \alpha = 0.85$ & $1- \alpha = 0.8$ \\ \hline
    \textbf{Split conformal prediction} & 0.838 (0.070) & 0.768 (0.088) & 0.684 (0.111) \\
    \textbf{Split TUC prediction} & 0.890 (0.035) & 0.836 (0.052) & 0.811 (0.001) \\ \hline
\end{tabular}
\caption{Empirical estimates of the minimum probability content of split conformal and split TUC prediction sets $\{\widehat{C}_{t, \alpha}\}_{t=1}^{100,000}$ over 100 replications.}
\label{tab:numerical_illustration}
\end{table}

The failure of split conformal prediction in a time-uniform sense stems from picking inappropriate sample quantiles to construct the prediction intervals. More specifically, the split conformal prediction algorithm for a dataset of size $n = n_1 + n_2$ involves using one portion of the observed data, say $\{Z_i\}_{i=1}^{n_1}$, to choose a 1-dimensional real-valued transformation $\widehat{R}_{n_1}: \mathcal{Z} \to \mathbb{R}$ and the other portion, say $\{Z_i\}_{i=n_1 + 1}^{n_1+n_2}$, to choose a particular sample quantile $\widehat{q}_{n_2, \alpha}$ of $\{\widehat{R}_{n_1}(Z_i)\}_{i=n_1 + 1}^{n_1 + n_2}$. The split conformal prediction set is then $\widehat{C}_{n, \alpha} := \{z: \widehat{R}_{n_1}(z) \leq \widehat{q}_{n_2, \alpha}\}$. In the same spirit as this approach, the TUC prediction sets are constructed by simply offsetting this choice of the sample quantile in order to ensure adequate time-uniform coverage. 

In Figure~\ref{fig:numerical_illustration}, our procedure involves computing the non-conformity score $\widehat{R}_{n_1}(\cdot)$ based on a fixed sample size $n_1$. This is also comparable to that of~\cite{avelin2023} and can be a source of sub-optimality, in general. For simplicity, we first discuss the construction of TUC and TUPAC prediction sets in this framework assuming a fixed independent dataset (i.e., held-out) for the non-conformity score. Later, we consider truly sequential prediction sets that can be constructed without any held-out data.

An added benefit of TUC prediction sets is that they can be built in an online fashion. This is in contrast to the split conformal prediction algorithm which processes data in batches. In particular, we propose an algorithm for constructing TUC prediction sets that is efficient in terms of both memory and computation – the analyst can update the TUC prediction set as data points are observed using fast online algorithms such as gradient descent and need not store data from previous epochs. In practice, the online nature of the proposed algorithm also lends itself to adaptivity when the distribution of the observations changes over time. That is, while a batch-trained model would need to be retrained from scratch in order to maintain its predictive accuracy (which is often costly and impractical), a model which is learned online is updated as more data are accumulated, allowing it to adjust to change points or shifts in the distribution; see Figure \ref{fig:changing_dist} for a simulated example of how the validity of our proposed online TUC prediction sets recovers when there is a change point in the observed data distribution. Our online uncertainty quantification methods are designed to wrap around online training algorithms and to inherit their adaptive properties. See \cite{gama2014}, \cite{lu2019}, and \cite{hoi2021} for a review of supervised learning algorithms for handling so-called concept drift.

\paragraph{Notation and assumptions} For the remainder of the paper, we will assume that an analyst observes a stream of independent data points $Z_1, Z_2, \dots$ where $Z_t \in \mathcal{Z}$ and $Z_t \sim \mu_Z$ for each $t = 1, 2, \dots$. In addition, we take $Z \in \mathcal{Z}$ to be an independent test point coming from the same distribution. The IID assumption is crucial because it allows us to reformulate the TUC and TUPAC goals in terms of the common probability measure $\mu_Z$. We use $T$ to denote a time chosen by the analyst at which to stop. In particular, $T$ is a random variable taking values in $\mathbb{N}^+$ whereas $t$ is taken to represent general fixed time points. We make no assumptions about the distribution of $T$.

\paragraph{Organization} The rest of the paper is organized as follows. In Section \ref{fixed_transformation_sec}, we introduce a simplified version of the TUC and TUPAC problems in which we assume access to some fixed transformation of the data, or equivalently, a held-out dataset for constructing non-conformity score. We present an algorithm for constructing time-uniform prediction sets and provide theoretical conformal and PAC guarantees in this setting. In Section \ref{dynamic_transformation_sec}, we relax the need to have access to a transformation ahead of time by allowing the transformation to change as we accrue more data. We present a time-uniform prediction set algorithm as well as theoretical results in this setting. In Section \ref{width_sec}, we show that our proposed time-uniform prediction sets are optimal in the sense that their widths converge to that of the oracle prediction set asymptotically. We present simulated and real-data experiments in Sections \ref{experiments_sec} and \ref{application_sec} that support the validity and utility of our approach. Finally, we conclude with a discussion of the results and propose directions for future work in Section \ref{discussion}.

Code for reproducing our experiments can be found at \hyperlink{https://github.com/kscharfs/time-uniform-conformal}{https://github.com/kscharfs/time-uniform-conformal}.

\section{Split time-uniform conformal and PAC prediction} \label{fixed_transformation_sec}

Recall that the time-uniform conformal (TUC) prediction goal is to construct a prediction set $\widehat{C}_{T, \alpha}$ based on the first $T$ observations with the guarantee of~\eqref{eq:tuc_goal}.
In an asymptotic sense, one can interpret \eqref{eq:tuc_goal} as \begin{equation*}
    \lim_{s \to \infty} \frac{1}{s}\sum_{t=1}^s \mathbbm{1}\left\{Z^*_t \in \widehat{C}_{T, \alpha}\right\} \geq 1 - \alpha,
\end{equation*} where $Z^*_1, \dots, Z^*_s$ are independent random variables from $\mu_Z$ independent of $\{\widehat{{C}}_{t,\alpha}\}_{t\ge 0}$. Thus, the TUC prediction guarantee is that if an analyst decides for some reason to stop and report $\widehat{C}_{T, \alpha}$ at time $T$, then approximately $100(1-\alpha)\%$ of all future observations will lie in $\widehat{C}_{T, \alpha}$.

Guarantee \eqref{eq:tuc_goal} differs from the usual conformal prediction guarantee \eqref{eq:conformal_goal} in that the number of observations is taken to be random and we have seen through Proposition \ref{prop:equiv_TUC_goal} that it is a much stronger guarantee. Despite this discrepancy, we can still use methods for the vanilla conformal approach to motivate our time-uniform approach. The key idea behind methods such as split conformal prediction is that we can reformulate \eqref{eq:conformal_goal} by introducing a 1-dimensional transformation of the data and then exploit a well-known result about the sample quantile to guarantee finite-sample coverage. In particular, notice that if we let $R: \mathcal{Z} \to \mathbb{R}$ be a real-valued transformation, then we can rewrite \eqref{eq:conformal_goal} as finding $\widehat{q}_{n, \alpha}$ based on $\{Z_i\}_{i=1}^n$ so that 
  $  \mathbb{P}\left(R(Z) \leq \widehat{q}_{n, \alpha}\right) \geq 1 - \alpha,$
since we can simply take $\widehat{C}_{n, \alpha} = \left\{z : R(z) \leq \widehat{q}_{n, \alpha}\right\}$. Defining $\widehat{F}_n(r) = n^{-1}\sum_{i=1}^n \mathbbm{1}\{R(Z_i) \leq r\}$ and choosing $\widehat{q}_{n, \alpha} = \inf\{r: n\widehat{F}_n(r)/(n+1) \geq 1-\alpha\}$ yields the desired coverage guarantee. Note that here we treat $R$ as fixed, but in practice, one could choose the transformation based on an independent dataset and obtain coverage results conditional on this dataset. 

Despite the fact that $T$ is random, we can apply the same idea by taking $\widehat{C}_{T, \alpha} = \left\{z: R(z) \leq \widehat{q}_{T, \alpha}\right\}$ given some fixed transformation $R$. It remains to choose an appropriate sample quantile of $\{R(Z_t)\}_{t=1}^T$ as $\widehat{q}_{T, \alpha}$. This motivates the following algorithm in which we abuse notation by dropping the indices $\alpha$ and $\delta$ in denoting the prediction sets and sample quantiles. We call this algorithm split TUC/TUPAC prediction. The algorithm is essentially the same as the split conformal prediction algorithm except that a more conservative sample quantile is chosen to construct the prediction set at each iteration.

\begin{algorithm} 
\caption{Split TUC/TUPAC/CS prediction} \label{fixed_anytime_alg}
\begin{algorithmic} 
\Require Transformation $R$; coverage level $1-\alpha \in (0, 1)$; $1-\delta \in (0, 1)$ if TUPAC variant. 
\Ensure A sequence of valid prediction sets $\widehat{C}_{1}, \widehat{C}_{2}, \dots$.
\State Let $t=0$. 
\While{the analyst wants to continue at time $t$}
    \State Set $t = t + 1$.
    \State Observe $Z_t$.
    \State Compute $R(Z_t)$.
    \State Compute $\widehat{q}_t$. \Comment{as defined in Theorems \ref{fixed_anytime_conformal_thm}, \ref{fixed_anytime_pac_thm}, or \ref{fixed_confidence_sequence_pac_thm}}.
    \State Report $\widehat{C}_{t} = \left\{z: R(z) \leq \widehat{q}_{t}\right\}$ as a prediction set for $Z_{t+1}$.
\EndWhile
\end{algorithmic}
\end{algorithm}

We consider two approaches for choosing $\widehat{q}_t$. First, we apply recent work on confidence sequences for a fixed quantile from \cite{howard2022} to achieve the TUPAC guarantee as in Theorem \ref{fixed_confidence_sequence_pac_thm}. Recall that a prediction set $\widehat{C}_{T, \alpha, \delta}$ is said to be $(\alpha, \delta)$-TUPAC if \begin{equation} \label{eq:anytime_pac_goal}\tag{TUPAC}
    \mathbb{P}\left\{\mathbb{P}\left(Z \in \widehat{C}_{T, \alpha, \delta} \vert Z_1, \dots, Z_T\right) \geq 1 - \alpha\right\} \geq 1 - \delta.
\end{equation} This is similar to the approach taken by \cite{avelin2023} except that they use confidence sequences for all $\alpha$ simultaneously; the resulting prediction sets are consequently larger than necessary if one is only seeking time-uniform coverage (with a fixed $\alpha$). Though confidence sequences can be used to obtain the TUPAC guarantee, it remains unclear how to achieve the TUC guarantee using such machinery. As such, we consider an alternative approach for choosing $\widehat{q}_t$ which relies on the fact that if $\widehat{q}_t$ is taken to be a sample quantile, then the probability content of the prediction set has the same distribution as a particular standard uniform order statistic. We then use a simple union bound over all $t \geq 0$ to choose an appropriate sample quantile. This approach yields a TUC prediction set when the sample quantile is chosen as in Theorem \ref{fixed_anytime_conformal_thm} and a TUPAC prediction set when the sample quantile is chosen as in Theorem \ref{fixed_anytime_pac_thm}.

\begin{theorem} \label{fixed_confidence_sequence_pac_thm}
Define \begin{equation*}
    \ell_t = \frac{1.4\log \log(2.1t) + \log\left({10}/{\delta}\right)}{t}, \quad u_t = 1.5\sqrt{\alpha(1-\alpha)\ell_t} + 0.8\ell_t
\end{equation*} and take $\widehat{q}_t = \inf\{r: \widehat{F}_t(r) \geq 1-\alpha + u_t\}.$ Then, Algorithm \ref{fixed_anytime_alg} produces a sequence of prediction sets satisfying \eqref{eq:anytime_pac_goal}.
\end{theorem}


Note that \cite{howard2022} also provides another approach for choosing $u_t$ in Theorem \ref{fixed_confidence_sequence_pac_thm}, albeit not in closed form. This alternative approach yields a larger $u_t = O(\sqrt{{\log(t)}/{t}})$ asymptotically, but often performs better in practice because the asymptotically optimal $u_t$ sacrifices precision at smaller sample sizes in order to gain tightness in the long run. Our union bounding approach (below) yields a similar asymptotically suboptimal $u_t$ that has similar practical advantages. In addition, we provide a mechanism for the analyst to choose sample sizes that are practically relevant; in Theorems \ref{fixed_anytime_conformal_thm} and \ref{fixed_anytime_pac_thm}, the function $h$ allocates the $\alpha$-budget in order to achieve time-uniform coverage. For instance, if one anticipates selecting a random time $T$ near $t=100$, $h(\cdot)$ could be designed to peak around 100, potentially using a Poisson$(100)$ probability mass function.

\begin{theorem} \label{fixed_anytime_conformal_thm}
Define \begin{equation*}
    u_t = \frac{4(2\alpha - 1)\log\left({1}/{h(t)}\right)}{3(t+3)} + \sqrt{\frac{2\alpha(1-
    \alpha)\log\left({1}/{h(t)}\right)}{t+2}} + \frac{1}{2}\sqrt{\frac{2\pi\alpha(1-\alpha)}{t+2}}\left(1 - \sum_{s=0}^{t_0} h(s)\right)
\end{equation*} and take $\widehat{q}_t = \inf\{r: \frac{t}{t+1}\widehat{F}_t(r) \geq 1-\alpha + u_t\}$ where $t_0$ is the smallest natural number such that $\widehat{q}_t$ is finite for all $t > t_0$ and $h: \mathbb{N} \to \mathbb{R}^{\geq 0}$ is a function satisfying $\sum_{t=0}^{\infty} h(t) = 1$.
Then, Algorithm \ref{fixed_anytime_alg} produces a sequence of prediction sets satisfying \eqref{eq:tuc_goal}. 
\end{theorem}

\begin{proof}
The full proof is given in Section \ref{fixed_anytime_conformal_proof}. The key idea is to apply Proposition \ref{prop:equiv_TUC_goal} in order to rewrite the TUC guarantee in terms of the expected minimum probability content over an infinite sequence of fixed-time prediction sets. From here, one can show that if $\widehat{q}_t \equiv \widehat{q}_{t, \alpha}$ is taken to be a sample quantile, then the probability content of the prediction set has the same distribution as a particular standard uniform order statistic. Using the well-known results about the Beta distribution, in particular Theorem 1 from \cite{skorski2023}, and a union bounding argument over time, we prove that $\mathbb{E}[\min_{t > t_0} U_{\lceil (t+1)(1-\alpha + u_t)\rceil:t}] \geq 1-\alpha$, where $U_{s:t}$ represents the $s$-th order statistics among the first $t$ in an infinite sequence of standard uniform random variables.  
\end{proof}

Theorem \ref{fixed_anytime_pac_thm} uses union bounding to give an analogous theoretical guarantee for Algorithm \ref{fixed_anytime_alg} in the context of the PAC framework. Let $\psi:[0,1]^2 \to\mathbb{R}$ represent the KL divergence between two Bernoulli random variables, i.e., 
$\psi(x, p) := p\log\left({p}/{x}\right) + (1 - p)\log\left({(1-p)}/{(1-x)}\right).$
\begin{theorem} \label{fixed_anytime_pac_thm} Define \begin{equation*}
    u_t = \frac{\log\left(\frac{1}{\delta}\left(1 - \sum_{s=0}^{t_0} h(s)\right)\right) - \log h(t)}{t+1}
\end{equation*} and take $\widehat{q}_t = \inf\{r: \psi(1-\alpha, \frac{t}{t+1}\widehat{F}_t(r)) \geq u_t \text{ and } t\widehat{F}_t(r) \geq (1-\alpha)(t+1)\}$ where $t_0$ is the smallest natural number such that $\widehat{q}_t$ is finite for all $t > t_0$ and $h: \mathbb{N} \to \mathbb{R}^{\geq 0}$ is a function satisfying $\sum_{t=0}^{\infty} h(t) = 1$. Then, Algorithm \ref{fixed_anytime_alg} produces a sequence of prediction sets satisfying \eqref{eq:anytime_pac_goal}.
\end{theorem}
\begin{proof}
The proof is given in Section \ref{fixed_anytime_pac_proof} and follows from applying Proposition \ref{prop:equiv_TUPAC_goal} and then using a similar argument to that for Theorem \ref{fixed_anytime_conformal_thm} so that a well-known result about the Beta distribution, in particular Proposition 2.1 from \cite{dumbgen1998}, can be applied to conclude the proof.
\end{proof}

Note that as $t\to\infty$, $\widehat{q}_t$ from Theorems \ref{fixed_confidence_sequence_pac_thm}, \ref{fixed_anytime_conformal_thm}, and \ref{fixed_anytime_pac_thm} approach the true $(1-\alpha)-$th quantile as expected. Since traditional split conformal prediction also relies asymptotically on this empirical quantile, we expect TUC/TUPAC and split conformal methods to exhibit similar width properties in the limit. This intuition is formalized in Section~\ref{width_sec}.

\section{Online time-uniform conformal and PAC prediction} \label{dynamic_transformation_sec}
While Algorithm~\ref{fixed_anytime_alg} provides sequential prediction sets satisfying TUC and TUPAC guarantees, it is bound by a fixed transformation that is implausible to be available in practice. For example, the optimal prediction set for a random variable $Z$ is given by $\{z:\, p_Z(z) \ge t_{\alpha}\}$ where $p_Z(z)$ is the density of $Z$ and $t_{\alpha}$ is some threshold to guarantee $1-\alpha$ coverage. With a fixed amount of data, one cannot learn $p_Z(\cdot)$ accurately, and so, Algorithm~\ref{fixed_anytime_alg} cannot be optimal no matter how many observations are used for quantile estimation.

To remedy this sub-optimality, we need the transformation to converge to the optimal one in an asymptotic sense. In addition, in settings where data are expensive to store, it is preferable to be able to construct prediction sets in an online fashion so that previous observations can be discarded. As such, we adapt the results from the previous section to allow for a dynamic transformation that can be updated as additional data points are observed.

Before proposing the algorithm, we will first provide some intuition. Consider splitting time into geometric epochs as 
    $\{t \geq 1\} = \cup_{k \geq 0} \{s: \eta^k \leq s < \eta^{k+1}\}$
for some $\eta > 1$. The algorithm is best understood in two stages. For time $t \in [\eta^k, \eta^{k+1})$, we use a transformation computed from data up to time $s < \eta^k$ and apply the sequential empirical quantile method, as described in Algorithm~\ref{fixed_anytime_alg}. Once data up to time $\eta^{k+1}$ is observed, we update the transformation using all data available up to $\eta^{k+1}$; this can be done in the backend through a sequential procedure (such as SGD) without storing all the data between $\eta^k$ and $\eta^{k+1}$.

In theory, this algorithm ensures TUC/TUPAC coverage guarantees. However, for $t \approx \eta^k, k\ge1$, the algorithm tends to produce significantly larger prediction sets due to insufficient data for accurately estimating the empirical quantile. To address this, we allow the analyst to revert to the transformation from the previous epoch if it results in a smaller prediction set. Figure \ref{fig:dynamic_transformation_illustration} and Algorithm~\ref{alg:dynamic_algo-illustration} below illustrate this dynamic setup.

\begin{figure}
    \centering
    \includegraphics[width=\textwidth]{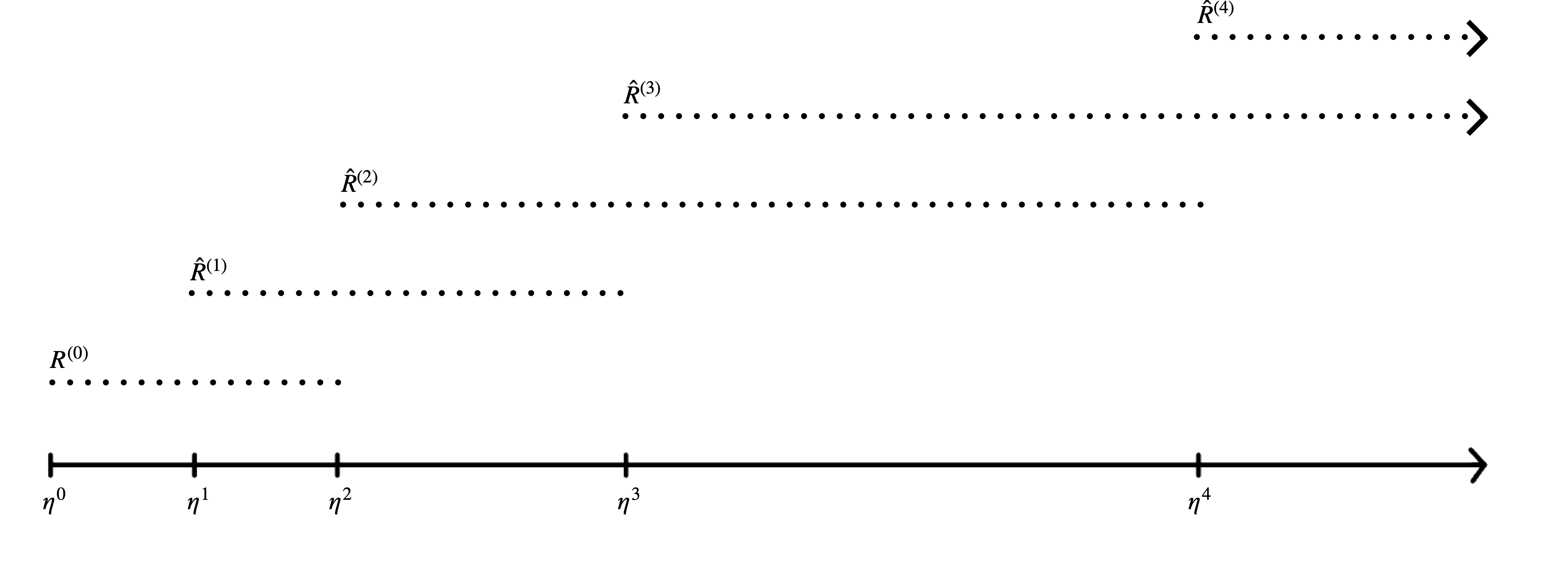}
    \caption{Illustration of dynamic updating of transformation for online  TUC and TUPAC prediction.}
    \label{fig:dynamic_transformation_illustration}
\end{figure}

Algorithm \ref{alg:dynamic_algo-illustration} formalizes this intuition for constructing time-uniform prediction sets using a dynamic transformation. We again abuse notation by suppressing the subscripts $\alpha$ and $\delta$ when denoting the prediction sets and sample quantiles. From Steps 4-6, it is clear that the transformation $\widehat{R}^{(k)}(\cdot)$ is computed based on $\lfloor \eta^k\rfloor$ observations. This, in turn, implies that the prediction set $\widehat{C}_{t,k-1}$ when evaluated for $t \in (\eta^k,\, \eta^{k+1}]$ has quantile estimated from $t - \lceil \eta^{k-1}\rceil + 1 \ge \lceil\eta^k\rceil - \lceil \eta^{k-1}\rceil + 1$ observations, which diverges as $k\to\infty$.

\begin{algorithm}
    \caption{Online TUC/TUPAC}\label{alg:dynamic_algo-illustration}
    \begin{algorithmic}[1]
        \Require Starting transformation $\widehat{R}^{(0)}$; algorithm $\mathcal{A}_t(z; R)$ for updating the transformation $R$ using data point $z$ for $t \geq 1$; epoch size $\eta > 1$; coverage level $1-\alpha \in (0, 1)$; $1 - \delta \in (0, 1)$ if PAC variant.
        \State Set $t=0$; $k=0$; $\widehat{R} = \widehat{R}^{(0)}$.
        \While{the analyst wants to continue at time $t$}
            \State Set $t = t + 1$.
            \If{$t \geq \eta^{k+1}$}
                \State Set $k = k + 1$.
                \State Set $\widehat{R}^{(k)} = \widehat{R}$.
            \EndIf
            \State Observe $Z_t$.
            \State Compute $\widehat{R}^{(k-1)}(Z_t)$ and $\widehat{R}^{(k)}(Z_t)$.
            \State Compute $\widehat{q}_{t, k-1}$ and $\widehat{q}_{t, k}$.  \Comment{as defined in Theorems \ref{dynamic_anytime_conformal_thm} or \ref{dynamic_anytime_pac_thm}}
            \State Report the smaller of $\widehat{C}_{t, k-1} = \{z: \widehat{R}^{(k-1)}(z) \leq \widehat{q}_{t, k-1}\}$ and $\widehat{C}_{t, k} = \{z: \widehat{R}^{(k)}(z) \leq \widehat{q}_{t, k}\}$ as a prediction set for $Z_{t+1}$.
            \State Update $\widehat{R} = \mathcal{A}_t(Z_t; \widehat{R})$.
        \EndWhile
    \end{algorithmic}
\end{algorithm}

In the following results, for $t \in [\eta^k, \eta^{k+2})$, let
\begin{equation*}
    \widehat{F}_{t, k}(r) = \frac{1}{s_{t, k}}\sum_{s=1}^{s_{t, k}}\mathbbm{1}\{\widehat{R}^{(k)}(Z_{\lceil\eta^k\rceil+s}) \leq r\},\quad\mbox{with}\quad s_{t, k} = t - (\lceil\eta^k\rceil -1),
\end{equation*} 
be the empirical CDF over the calibration set through time $t \in [\eta^k, \eta^{k+2})$. The following theorem shows that the proposed algorithm produces a prediction set that achieves the TUC prediction guarantee.

\begin{theorem} \label{dynamic_anytime_conformal_thm}
Define \begin{align*}
    u_{t,k} &= \frac{4(2\alpha - 1)\log \left(\frac{1}{h(s_{t, k})}\right)}{3(s_{t, k}+3)} + \sqrt{\frac{2\alpha(1-\alpha)}{s_{t, k}+2}}\left(\sqrt{\log \frac{1}{h(s_{t, k})}} + \sqrt{\log\left(\sum_{t'=t_k+1}^{\lfloor \eta^{k+2}\rfloor}\frac{h(s_{t', k})}{g(k)}\right)} + \frac{\sqrt{\pi}}{2}\right)
\end{align*} and take $\widehat{q}_{t, k} = \inf\{r: \frac{s_{t,k}}{s_{t,k} + 1}\widehat{F}_{t, k}(r) \geq 1-\alpha+u_{t, k}\}$ where $t_k$ is the smallest natural number such that $\widehat{q}_{t, k}$ for all $t \in (t_k, \eta^{k+2})$ and $h, g: \mathbb{N} \to \mathbb{R}^{\geq 0}$ are functions satisfying $\sum_{t=0}^{\infty} h(t) = \sum_{t=0}^{\infty} g(t) = 1$. Then, Algorithm \ref{alg:dynamic_algo-illustration} produces a sequence of prediction sets satisfying \eqref{eq:tuc_goal}.
\end{theorem}
The proof is given in Section \ref{dynamic_anytime_conformal_proof} and follows a similar argument used to prove Theorem \ref{fixed_anytime_conformal_thm}.
An analogous result in the PAC setting is given in Theorem \ref{dynamic_anytime_pac_thm}. 

\begin{theorem} \label{dynamic_anytime_pac_thm}
Define \begin{align*}
    u_{t,k} &= (s_{t, k} + 1)^{-1}{\log\left(\sum_{t' = t_k + 1}^{\lfloor \eta^{k+2}\rfloor}\frac{ h(s_{t', k})}{\delta g(k)h(s_{t, k})}\right)}
\end{align*} and take $\widehat{q}_{t, k} = \inf\{r: \psi(1-\alpha, \frac{s_{t,k}}{s_{t,k} + 1}\widehat{F}_{t, k}(r)) \geq u_{t, k} \text{ and } s_{t, k}\widehat{F}_{t, k}(r) \geq (1-\alpha)(s_{t, k} + 1)\}$ where $t_k$ is the smallest natural number such that $\widehat{q}_{t, k}$ is finite for all $t \in (t_k, \eta^{k+2})$ and $h, g: \mathbb{N} \to \mathbb{R}^{\geq 0}$ are functions satisfying $\sum_{t=0}^{\infty} h(t) = \sum_{t=0}^{\infty} g(t) = 1$. Then, Algorithm \ref{alg:dynamic_algo-illustration} produces a sequence of prediction sets satisfying \eqref{eq:anytime_pac_goal}.
\end{theorem}
The proof is given in Section \ref{dynamic_anytime_pac_proof} and follows a similar argument to that used to prove Theorem \ref{dynamic_anytime_conformal_thm}.

\paragraph{Practical considerations.} While Theorems \ref{dynamic_anytime_conformal_thm} and \ref{dynamic_anytime_pac_thm} demonstrate the theoretical viability of applying Algorithm \ref{alg:dynamic_algo-illustration} to construct time-uniform prediction sets, there are several important practical considerations worth mentioning. Algorithm \ref{alg:dynamic_algo-illustration} can be implemented as an online algorithm as long as the transformation can be updated in an online fashion. For example, suppose the analyst observes data points $Z_t = (X_t, Y_t)$, $t = 1, 2, \dots$ where $X_t$ are the observed features or covariates at time $t$ and $Y_t$ is the outcome at time $t$. A common choice of transformation is a loss function measuring the discrepancy between the outcome $Y_t$ and a prediction of the outcome based on the covariates $X_t$ (e.g., the absolute loss or squared loss). In particular, suppose the analyst sets the transformation introduced at the $k$th epoch to be $\widehat{R}^{(k)}(x, y) = \ell(y, \widehat{\mu}^{(k)}(x))$ where $\ell$ is a loss function and $\widehat{\mu}^{(k)}$ is a predictor trained on the dataset $\{(X_t, Y_t)\}_{t=1}^{\lfloor\eta^k - 1\rfloor}$. If the predictor can be trained in an online fashion (e.g., via stochastic gradient descent), then the transformation can also be updated online, which eliminates the need to store data from previous epochs. This property can be extended to other settings, such as conformalized quantile regression \citep{romano2019}. 
See \cite{hu2012} for an online quantile regression algorithm which involves descending along the gradient of a version of the pinball loss function. In addition, one may also want to consider using (online) algorithms which forget or downweight old observations in favor of newer data for training the predictor or quantile regressor. Such unlearning algorithms have been shown to be able to adapt more quickly to shifting distributions, making them attractive in practice; see \cite{you2016}, \cite{liu2016}, \cite{zhang2017}, \cite{thudi2022}, and \cite{yang2025} for examples.

\section{Oracle inequalities for the ``size'' of prediction sets} \label{width_sec}

To evaluate the optimality of our proposed online prediction sets, we compare their widths to those of the corresponding oracle prediction sets, as was done in~\cite{lei2018distribution}. We consider the special case of applying Algorithm \ref{alg:dynamic_algo-illustration} with the transformation introduced at the $k^{\text{th}}$ epoch taken to be $\widehat{R}^{(k)}\left(z\right) := R(z ; \widehat{\xi}^{(k)})$, where $R$ is a fixed transformation and $\widehat{\xi}^{(k)}$ contains nuisance functions which depend on the data $\left\{Z_s\right\}_{s=1}^{\lfloor \eta^k \rfloor}$. For example, in the case of regression where an analyst observed IID covariate-response pairs $Z_t = (X_t, Y_t)$, $t=1, 2, \dots$, one could take $\widehat{\xi}^{(k)} = \widehat{\mu}^{(k)}$ to be an estimate of the regression function $\mu(x) = \mathbb{E}[Y \vert X=x]$ based on the data $\left\{(X_s, Y_s)\right\}_{s=1}^{\lfloor \eta^k \rfloor}$ so that $R((x, y); \widehat{\xi}^{(k)}) = |\widehat{\mu}^{(k)}(x) - y|$. As another example, in order to obtain an online time-uniform version of conformalized quantile regression, one could take $\widehat{\xi}^{(k)} = (\widehat{q}_{\alpha/2}^{(k)}, \widehat{q}_{1 - \alpha/2}^{(k)})$, where $\widehat{q}_{\alpha/2}^{(k)}$ and $\widehat{q}_{1-\alpha/2}^{(k)}$ are estimates of the ${\alpha}/{2}$ and $1-{\alpha}/{2}$ quantiles of the conditional distribution of $Y$ given $X$ based on $\left\{(X_s, Y_s)\right\}_{s=1}^{\lfloor \eta^k \rfloor}$, respectively, and set $R((x, y); \widehat{\xi}^{(k)}) = \max\{\widehat{q}^{(k)}_{\alpha/2}(x) - y, y - \widehat{q}^{(k)}_{1-\alpha/2}(x)\}$.

From Algorithm \ref{alg:dynamic_algo-illustration}, the online TUC prediction set at time $t \in \left[\eta^k, \eta^{k+1}\right)$ is \begin{equation*}
    \widehat{C}_{t, \alpha} = \widehat{C}_{t, k^*, \alpha} = \left\{z : R(z; \widehat{\xi}^{(k^*)}) \leq \widehat{q}_{t, k^*, \alpha}\right\},
\end{equation*} 
where $k^* = \argmin_{j \in \{k-1, k\}} L(\widehat{C}_{t, j, \alpha})$, $L$ denotes the Lebesgue measure (or any measure of ``size'' of the set), and $\widehat{q}_{t, j, \alpha}$ is a particular sample quantile of the learned transformation $R(\cdot; \widehat{\xi}^{(j)})$ applied to the calibration set $\left\{Z_s\right\}_{\left\lceil \eta^j \right\rceil}^t$ for $j \in \{k-1, k\}$. 

We compare the ``size'' of the prediction set reported to an oracle prediction set. The oracle prediction set can be defined as the one obtained at $t = \infty$, i.e., using $\xi^*$ the ``population limit'' of $\xi^{(k)}$ as $k\to\infty$ and $q^*_{\alpha}$ the $(1-\alpha)$-th quantile of $R(Z; \xi^*)$. Formally, we define the oracle prediction set to be \begin{equation*}
    C^*_{\alpha} = \left\{z: R\left(z; \xi^*\right) \leq q^*_{\alpha}\right\}.
\end{equation*} 

Ideally, as we accumulate more data, $\widehat{C}_{t, \alpha}$ should get very close to $C^*_{\alpha}$. As such, we introduce three conditions which together with the implicit assumption that the data are independent and identically distributed, will imply that the width of the online TUC prediction band converges to that of the oracle prediction set as is shown in Theorem \ref{thm:optimal_width}. 

\begin{enumerate}[label=(\textbf{A\arabic*})] 
    \item For some metric $d(\cdot, \cdot)$ on the space of nuisance functions ($\xi$), \label{eq:Lipschitz-score}
    \begin{equation*}
        \sup_z \left|R\left(z; \xi\right) - R(z; \xi')\right| \le d(\xi,\xi')\quad\mbox{for all}\quad \xi, \xi'.
    \end{equation*} 
\end{enumerate}

\begin{theorem} \label{thm:optimal_width}
Under assumption~\ref{eq:Lipschitz-score}, the online TUC prediction set $\widehat{C}_{t, \alpha}$ (for any fixed $t \ge 1$) satisfies 
\begin{equation*}
    L(\widehat{C}_{t, \alpha}) \le   L\left(C^*_{\alpha}\right) + g\left(2d\left(\xi^*,\,\widehat{\xi}^{(k-1)}\right) + \max_{\omega\in\{\pm 1\}}\left|q^{*}_{\alpha} - q^*_{\alpha + u_{t, k-1} + \omega\sqrt{\eta\log(2/\kappa)/(2t(\eta - 1))}}\right|\right),
\end{equation*} 
with probability at least $1 - \kappa$, where $g(\epsilon) = L\left(\left\{z: \left|R(z; \xi^*) - q^*_{\alpha}\right| < \epsilon\right\}\right)$. In particular, if the density of $R(Z; \xi^*)$ is bounded away from zero in a neighborhood of $q^*_{\alpha}$ and $g(\epsilon)\asymp\epsilon$ as $\epsilon\to0$, then for large enough $t$, we get
\[
L(\widehat{C}_{t, \alpha}) \le   L\left(C^*_{\alpha}\right) + O_{\mathbb{P}}(1)d\left(\xi^*, \widehat{\xi}^{(k-1)}\right) + O_{\mathbb{P}}(1)\left(\frac{1}{\sqrt{t}} + u_{t, k-1}\right).
\]
\end{theorem}
\begin{proof}
The proof is given in Section \ref{optimal_width_proof}.
\end{proof}
The proof follows the basic structure of the oracle inequality for split conformal prediction. Assumption~\ref{eq:Lipschitz-score} is rather mild and is satisfied in many settings leading to optimality of the volume of the prediction sets. For example, in the context of regression data, $z = (x, y)$ and the CQR score has $\xi^*(z) = (q_{\alpha/2}(x),\,q_{1-\alpha/2}(x))$ and $R(z; \xi^*) = \{\xi_1^*(x, y) - y,\, y - \xi_2^*(x, y)\}$. This implies that
\[
|R(z; \xi) - R(z; \xi')| \le \max\left\{\|\xi_1 - \xi'_1\|_{\infty},\, \|\xi_2 - \xi_2\|_{\infty}\right\}.
\]
Therefore, $d(\xi^*,\, \widehat{\xi}^{(k-1)})$ converges in probability to zero if the conditional quantiles are consistently estimated. Additionally, it follows from the results of~\cite{sesia2020comparison} that $C^*_{\alpha}$ is the narrowest {\em symmetric} prediction interval that has valid coverage conditional on the value of covariates. (Here symmetric means that the true response is equally likely to be overpredicted or underpredicted.) In the regression context, $L(\widehat{C}_{t,\alpha})$ measures the ``width'' conditional on $x$, i.e., with CQR score, we have $L(\widehat{C}_{t,\alpha}) = |\widehat{\xi}^{(k^*)}_2(x) - \widehat{\xi}^{(k^*)}_1(x)| + \widehat{q}_{t,k^*,\alpha}$.

From the proof of Theorem~\ref{thm:optimal_width}, it follows that $\widehat{q}_{t,k^*,\alpha}$ converges in probability to $q^*_{\alpha}$ (as $t\to\infty$) and this combined with consistency of $\widehat{\xi}^{(k-1)}$ to $\xi^*$ implies that the reported TUC and TUPAC prediction sets also satisfy asymptotic conditional coverage guarantee; this follows the same argument as Corollary 1 of~\cite{sesia2020comparison}.

\section{Experiments} \label{experiments_sec}

\subsection{Split time-uniform experiments}

In order to study the properties of our proposed algorithms for constructing time-uniform prediction sets, we conduct an empirical analysis under a simple simulation design based on an IID dataset of covariate-outcome pairs $(X_t, Y_t)$. We construct the true regression function $\mu(x) = \mathbb{E}[Y_t \vert X_t = x]$ to be linear in $x$ and draw 10-dimensional covariates $X_t \sim N(0, I)$ and errors $\epsilon_t := Y_t - \mu(X_t) \sim N(0, 1)$. We split the simulated dataset of size 150,000 into 3 equal parts. The first part is used to train a linear regression model $\widehat{\mu}$ for constructing the transformation $R(x, y) = \left|y - \widehat{\mu}(x)\right|$. We then use the second part to construct prediction sets with $\alpha = \delta = 0.1$ using the split conformal, CS (Theorem \ref{fixed_confidence_sequence_pac_thm}), TUC (Theorem \ref{fixed_anytime_conformal_thm}), TUPAC (Theorem \ref{fixed_anytime_pac_thm}), and SIPI \citep{avelin2023} algorithms. For the split CS method based on the work of \cite{howard2022}, we use the Beta-Binomial mixture approach implemented in the Confseq package \citep{howard2021} which has been shown to perform better in practice than the asymptotically optimal approach described in Theorem \ref{fixed_confidence_sequence_pac_thm}. For the split TUC and TUPAC intervals, we take the $h$ function to be the PMF of a discretized log Normal random variable with mean 11 and variance 1. We compare the prediction sets constructed using the various algorithms using the third split. In particular, we evaluate the coverage and width of the intervals as shown in Figure \ref{fig:fixed_simulation}. We observe that while the split conformal prediction intervals have nearly appropriate coverage at each time $t$, they do not have valid time-uniform coverage. In addition, the widths of all the prediction sets appear to approach that of the oracle parametric prediction interval as $t$ increases. The rate of convergence is slower for the time-uniform prediction intervals than the split conformal prediction intervals, but this is a necessary tradeoff in order to obtain time-uniform coverage. On the other hand, the width of the CS, TUC, and TUPAC intervals converges faster than the width of the SIPIs since they require extra width to guarantee uniform coverage over the miscoverage level $\alpha$.

\begin{figure}[h]
    \centering
    \includegraphics[width=\textwidth]{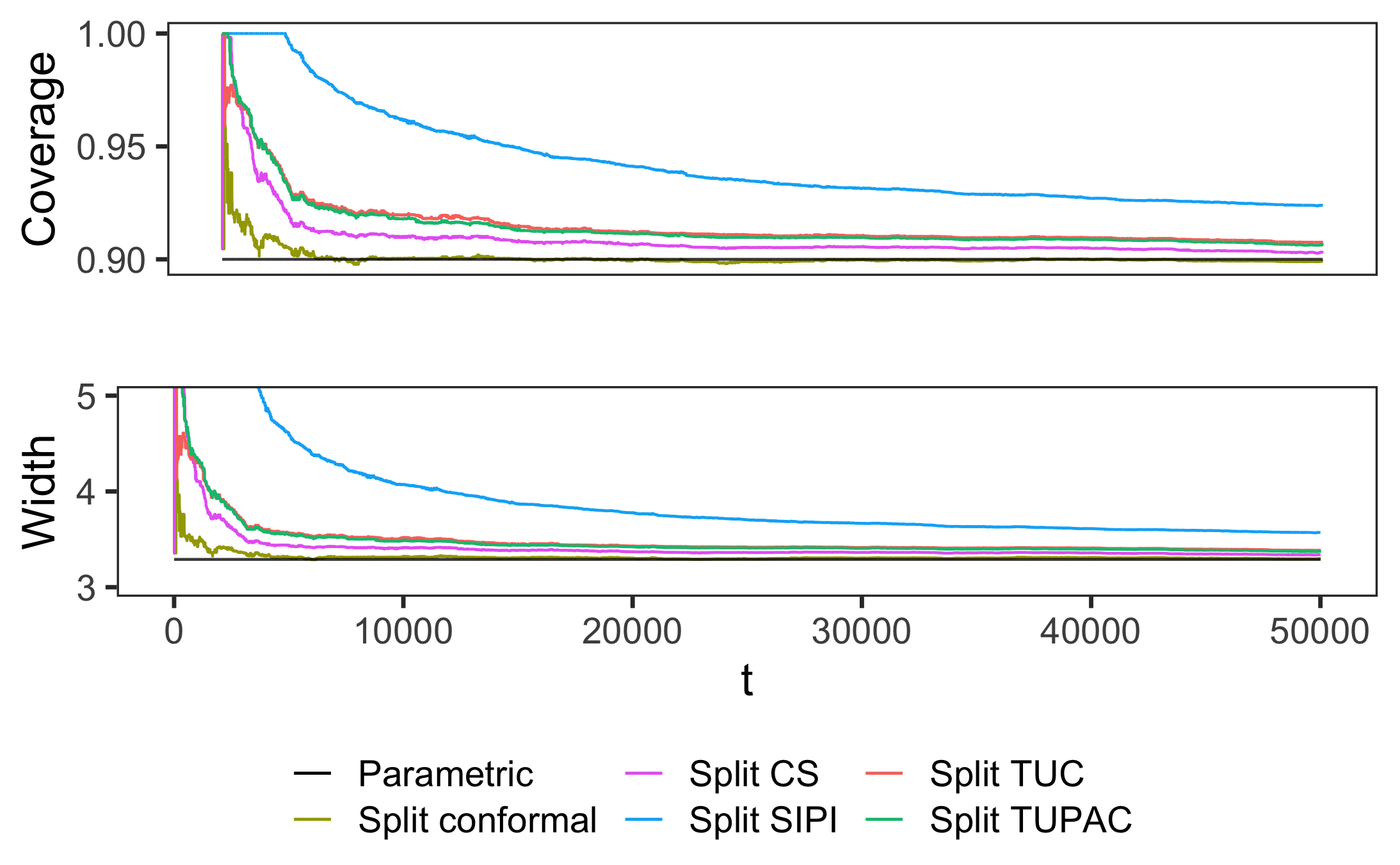}
    \caption{The coverage and width of prediction sets when split conformal, CS, TUC, TUPAC, and SIPI based on linear regression are applied to a simulated dataset described in Section \protect\ref{experiments_sec} are shown above.}
    \label{fig:fixed_simulation}
\end{figure}

\subsection{Online time-uniform experiments}

Though the theoretical guarantees of TUC prediction sets are specific to the IID setting, we suspect that online TUC prediction sets will exhibit attractive adapative behavior when the distribution of the observations changes over time. This is because the transformation is updated as more data is accrued which allows it to change over time. To investigate the adaptivity of the online TUC prediction sets, we return to the toy example described in Section \ref{introduction_sec}. Instead of drawing each observation from a standard Normal distribution, we draw the first 50,000 observations from a standard Normal distribution and the last 50,000 observations from a Normal distribution with mean and variance 1. Thus, the distribution of the observations shifts at $t = 50,000$. In applying the online TUC algorithm, we use the transformation $\widehat{R}^{(k)}(z) = \left|z - \bar{Z}_k\right|$ where $\bar{Z}_k = \frac{1}{\lceil\eta^k\rceil - 1}\sum_{t=1}^{\lceil \eta^k\rceil - 1} Z_t$ is the average of the observations prior to the start of the $k$th epoch, take $\eta=2$, and choose the $h$ and $g$ functions to be PMF of discretized log Normal random variables with means 11 and log(16), respectively, and variance 1. Figure \ref{fig:changing_dist} shows the resulting probability content of the online TUC prediction sets over 100,000 observations. We see that the probability content degrades from the desired 0.9 level at $t=50,000$, but subsequently recovers as more data coming from the shifted distribution are observed. 

\begin{figure}
    \centering
    \includegraphics[width=0.75\textwidth]{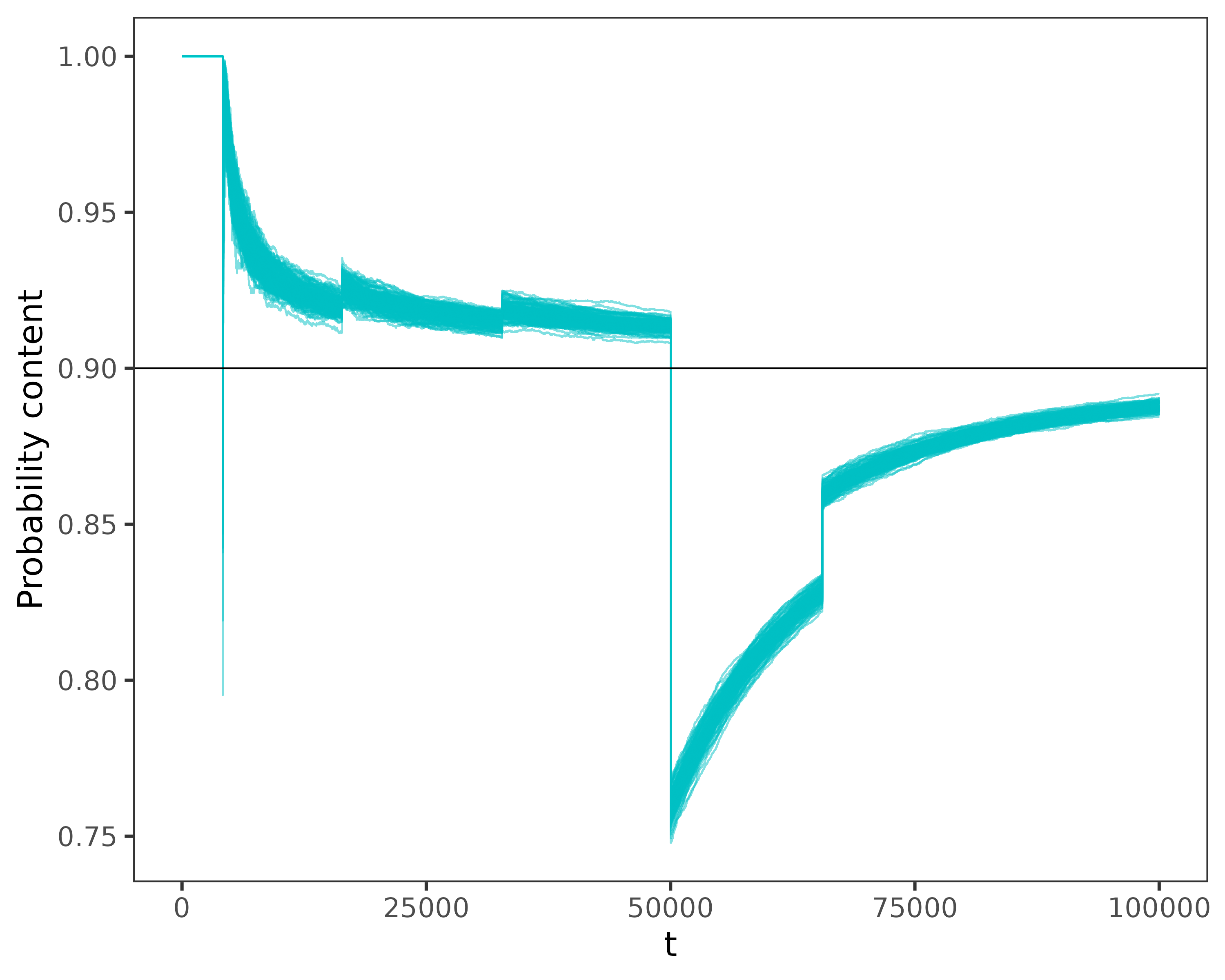}
    \caption{The true probability content of 90\% online TUC prediction intervals $\{\widehat{C}_{t, 0.1}\}_{t=1}^{100,000}$ constructed using Algorithm \protect\ref{alg:dynamic_algo-illustration} over 120 replications is shown above.}
    \label{fig:changing_dist}
\end{figure}

\section{Real data application to spam detection} \label{application_sec}

We apply our TUC and TUPAC prediction algorithms to detect spam emails. In particular, it is of interest to construct a prediction set which contains the true classication of an email (i.e. spam or not spam) with high probability each time an email is received. To illustrate the utility of TUC and TUPAC prediction in this setting, we apply our algorithms to a collection of 4601 emails \citep{dua2019}. Each email is classified as either spam or not spam and has been preprocessed to produce 54 features measuring frequencies of particular words and characters within the email and 3 features measuring the average, longest, and total length of sequences of consecutive capital letters within the email. These features are known to be useful when constructing a personalized spam filter. 

To apply and evaluate Algorithm \ref{fixed_anytime_alg} in this setting, we split the data into 3 equal parts. The first part is used to train a probabilistic classifier; we consider logistic regression, random forest, and a Superlearner ensemble of generalized linear models, generalized additive models, random forest, and kernel support vector machines. We use these classifiers to construct the absolute residual non-conformity score. Treating the classifier trained on the first split as fixed, we then apply Algorithm \ref{fixed_anytime_alg} on the second split to produce CS, TUC, and TUPAC prediction intervals with $\alpha = \delta = 0.1$. We take the $h$ function in Algorithm \ref{fixed_anytime_alg} to be the PMF of a discretized log Normal random variable with mean 6 and variance 1. We also apply the split conformal prediction algorithm as a baseline and note that the SIPI method from \cite{avelin2023} produced completely noninformative prediction sets and we thus omit them from our results. We use the third split to evaluate the coverage and width of the prediction intervals as shown in Figure \ref{fig:fixed_application}. We observe that the split CS, TUC, and TUPAC prediction intervals achieve time-uniform coverage in contrast to the split conformal prediction intervals. However, this time-uniform coverage comes at a cost – the CS, TUC, and TUPAC prediction intervals tend to be noninformative more often than split conformal prediction intervals in the sense that they contain both outcomes, spam and not spam. As $t$ increases, the frequency of prediction intervals containing both spam and not spam among observations in the third split diminishes to 0 for all methods.

\begin{figure}[h]
    \centering
    \includegraphics[width=0.9\textwidth]{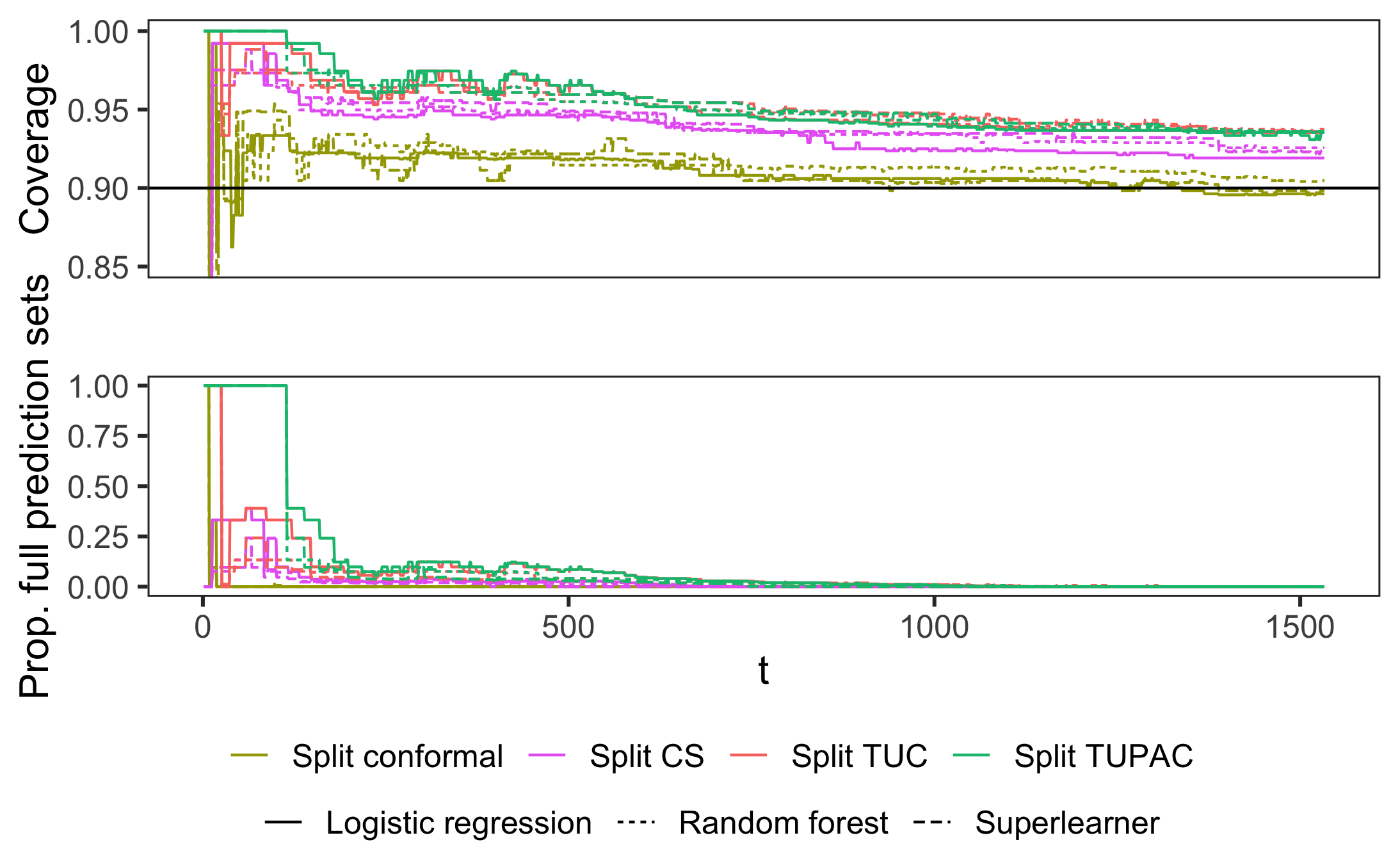}
    \caption{The coverage and proportion of noninformative prediction sets when split conformal, CS, TUC, and TUPAC prediction algorithms based on logistic regression, random forest, and a Superlearner ensemble are applied to the problem of spam detection are shown above.}
    \label{fig:fixed_application}
\end{figure}

We also construct online TUC and TUPAC prediction sets (Algorithm \ref{alg:dynamic_algo-illustration}) based on a logistic regression classifier trained via gradient descent. We again use the absolute residual non-conformity score. We take $\eta = 2$ and the functions $h$ and $g$ to be PMFs of a discretized log Normal random variable with means 6 and $\ln(8)$ and variance 1, respectively. We use $2/3$ of the dataset to construct the prediction intervals and $1/3$ to evaluate their coverage and width. Figure \ref{fig:dynamic_application} shows that the online TUC and TUPAC prediction sets achieve time-uniform coverage. In addition, the proportion of noninformative prediction intervals decays to 0 as $t$ increases for both methods.

\begin{figure}[h]
    \centering
    \includegraphics[width=0.9\textwidth]{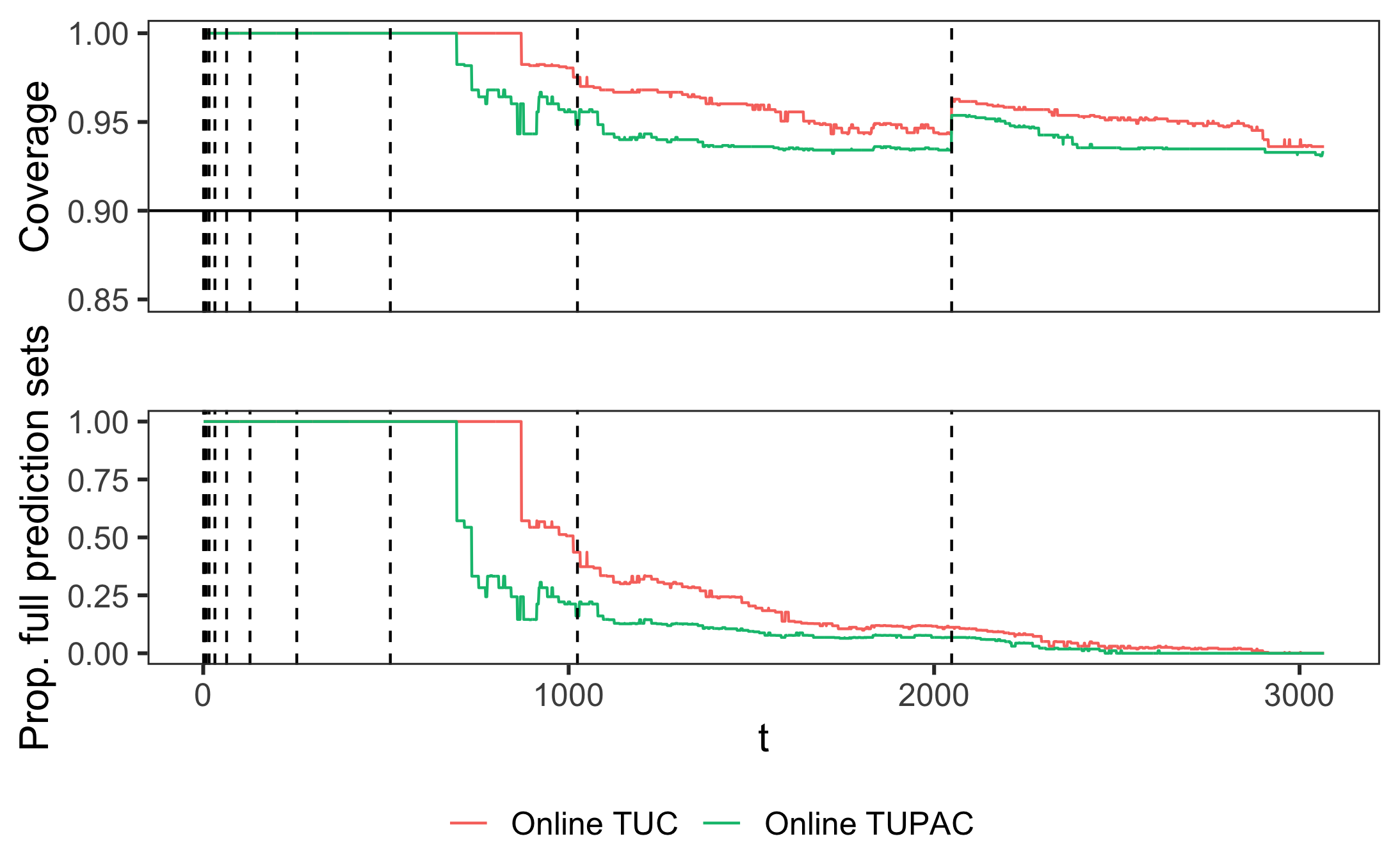}
    \caption{The coverage and proportion of noninformative prediction sets when online TUC and TUPAC prediction algorithms based on logistic regression trained via gradient descent are applied to the problem of spam detection. Each dashed vertical line represents the beginning of a new epoch.}
    \label{fig:dynamic_application}
\end{figure}

\section{Discussion} \label{discussion}
This paper has introduced variants of the conformal and PAC prediction sets that can be built online and relax the requirement of knowing the sample size ahead of time. In particular, we proposed two algorithms for constructing such anytime-valid distribution-free prediction sets; the first requires access to a fixed transformation while the second allows the transformation to be updated online as data is accumulated sequentially. 

To the best of our knowledge, this paper provides the first formalization of the time-uniform conformal and PAC prediction problems and thus opens the door to several interesting directions for future work. An important extension of our work is to settings where the data are not IID; in particular, since we are motivated by settings where data are observed sequentially, the distribution of the data will likely change over time in many cases. While we have shown that our online algorithms are adaptive in practice, it remains to explore theoretical guarantees in such scenarios. The conformal prediction problem under covariate shift has been explored in depth by \cite{tibshirani2019, lei2021, yang2022}. In particular, \cite{lei2021} provides a means for constructing prediction intervals for the individual treatment effect through the lens of covariate shift; it would be interesting to pursue whether our time-uniform conformal prediction method could be extended to construct prediction intervals for such causal quantities when the sample size is not known ahead of time. Another direction for future work is to explore whether the width of the proposed time-uniform prediction sets can be improved in finite samples. 

 \bibliography{references}

@Article{briceno2020,
   author={Brice{\~n}o, Javier and Ayll\'{o}n, Mar{\'i}a Dolores and Ciria, R{\'u}ben},
   title={Machine-learning algorithms for predicting results in liver transplantation: the problem of donor-recipient matching},
   journal={Current Opinion in Organ Transplantation},
   year={2020},
   volume={25},
   pages={406--411},
   month={8}
}

@article{lei2013,
    author={Lei, Jing and Robins, James and Wasserman, Larry},
    title={Distribution Free Prediction Sets},
    journal={Journal of the American Statistical Association},
    year={2013},
    volume={108},
    pages={278--287},
    month={3}
}

@article{lei2014,
    author={Lei, Jing and Wasserman, Larry},
    title={Distribution-free prediction bands for nonparametric regression},
    journal={Journal of the Royal Statistical Society Series B},
    year={2014},
    volume={76},
    pages={71--96}
}

@article{gibbs2021adaptive,
  title={Adaptive conformal inference under distribution shift},
  author={Gibbs, Isaac and Cand{\'e}s, Emmanuel},
  journal={Advances in Neural Information Processing Systems},
  volume={34},
  pages={1660--1672},
  year={2021}
}

@article{gibbs2024conformal,
  title={Conformal inference for online prediction with arbitrary distribution shifts},
  author={Gibbs, Isaac and Cand{\'e}s, Emmanuel},
  journal={Journal of Machine Learning Research},
  volume={25},
  number={162},
  pages={1--36},
  year={2024}
}

@article{angelopoulos2024conformal,
  title={Conformal pid control for time series prediction},
  author={Angelopoulos, Anastasios and Cand{\'e}s, Emmanuel and Tibshirani, Ryan J},
  journal={Advances in neural information processing systems},
  volume={36},
  year={2024}
}

@inproceedings{bhatnagar2023improved,
  title={Improved online conformal prediction via strongly adaptive online learning},
  author={Bhatnagar, Aadyot and Wang, Huan and Xiong, Caiming and Bai, Yu},
  booktitle={International Conference on Machine Learning},
  pages={2337--2363},
  year={2023},
  organization={PMLR}
}

@inproceedings{zaffran2022adaptive,
  title={Adaptive conformal predictions for time series},
  author={Zaffran, Margaux and F{\'e}ron, Olivier and Goude, Yannig and Josse, Julie and Dieuleveut, Aymeric},
  booktitle={International Conference on Machine Learning},
  pages={25834--25866},
  year={2022},
  organization={PMLR}
}

@article{massart1990,
    author = {Massart, Pascal},
    title = {The Tight Constant in the Dvoretzky-Kiefer-Wolfowitz Inequality},
    volume = {18},
    journal = {The Annals of Probability},
    number = {3},
    publisher = {Institute of Mathematical Statistics},
    pages = {1269 -- 1283},
    year = {1990},
    doi = {10.1214/aop/1176990746},
    URL = {https://doi.org/10.1214/aop/1176990746}
}

@article{gyorfi2019,
    author={Gy{\"o}rfi, L{\'a}szl{\'o} and Walk, Harro},
    title={Nearest neighbor based conformal prediction},
    volume={63},
    journal={Pub. Inst. Stat. Univ. Paris},
    year={2019}
}

@misc{yang2021,
    author={Yang, Yachong and Kuchibhotla, Arun Kumar},
    title={Finite-sample efficient conformal prediction},
    year={2021},
    url={https://arxiv.org/abs/2104.13871},
    publisher={arXiv}
}

@article{sesia2020comparison,
  title={A comparison of some conformal quantile regression methods},
  author={Sesia, Matteo and Cand{\'e}s, Emmanuel},
  journal={Stat},
  volume={9},
  number={1},
  pages={e261},
  year={2020},
  publisher={Wiley Online Library}
}

@article{lei2018distribution,
  title={Distribution-free predictive inference for regression},
  author={Lei, Jing and G’Sell, Max and Rinaldo, Alessandro and Tibshirani, Ryan J and Wasserman, Larry},
  journal={Journal of the American Statistical Association},
  volume={113},
  number={523},
  pages={1094--1111},
  year={2018},
  publisher={Taylor \& Francis}
}

@Book{guttman1970,
    author = {Guttman, Irwin},
    language = {eng},
    lccn = {74589969},
    publisher = {Griffin},
    series = {Griffin's statistical monographs and courses ; no. 26.},
    title = {Statistical tolerance regions: classical and Bayesian},
    year = {1970},
    address = {London}
}

@Book{krishnamoorthy2009,
    author={Krishnamoorthy, Kalimuthu and Matthew, Thomas},
    title={Statistical Tolerance Regions: Theory, Applications, and Computation},
    publisher={Wiley},
    year={2009},
    month={3}
    }

@Book{vovk2005,
    author={Vovk, Vladimir and Gammerman, Alexander and Shafer, Glenn},
    title={Algorithmic Learning in a Random World},
    publisher={Springer New York, NY},
    year={2005}}

@article{badue2021,
    title = {Self-driving cars: A survey},
    journal = {Expert Systems with Applications},
    volume = {165},
    pages = {113816},
    year = {2021},
    issn = {0957-4174},
    doi = {https://doi.org/10.1016/j.eswa.2020.113816},
    url = {https://www.sciencedirect.com/science/article/pii/S095741742030628X},
    author = {Claudine Badue and R{\^a}nik Guidolini and Raphael Vivacqua Carneiro and Pedro Azevedo and Vinicius B. Cardoso and Avelino Forechi and Luan Jesus and Rodrigo Berriel and Thiago M. Paix{\~a}o and Filipe Mutz and Lucas {de Paula Veronese} and Thiago Oliveira-Santos and Alberto F. {De Souza}}
}

@misc{angwin2016, 
    title={Machine bias}, 
    url={https://www.propublica.org/article/machine-bias-risk-assessments-in-criminal-sentencing}, 
    journal={ProPublica}, 
    author={Angwin, Julia and Larson, Jeff and Kirchner, Lauren and Mattu, Surya}, 
    year={2016}, 
    month={5}
}

@misc{yang2022,
  doi = {10.48550/ARXIV.2203.01761},
  
  url = {https://arxiv.org/abs/2203.01761},
  
  author = {Yang, Yachong and Kuchibhotla, Arun Kumar and Tchetgen, Eric Tchetgen},
  
  title = {Doubly Robust Calibration of Prediction Sets under Covariate Shift},
  
  publisher = {arXiv},
  
  year = {2022},
  
  copyright = {Creative Commons Attribution Share Alike 4.0 International}
}

@InProceedings{tibshirani2019,
    author={Tibshirani, Ryan J and Barber, Rina Foygel and Cand{\'e}s, Emmanuel and Ramdas, Aaditya},
    title={Conformal Prediction under Covariate Shift},
    year={2019},
    booktitle={Advances in neural information processing systems}}

@article{lei2021,
    author = {Lei, Lihua and Cand{\'e}s, Emmanuel},
    title = {Conformal inference of counterfactuals and individual treatment effects},
    journal = {Journal of the Royal Statistical Society: Series B (Statistical Methodology)},
    volume = {83},
    number = {5},
    pages = {911-938},
    doi = {https://doi.org/10.1111/rssb.12445},
    url = {https://rss.onlinelibrary.wiley.com/doi/abs/10.1111/rssb.12445},
    eprint = {https://rss.onlinelibrary.wiley.com/doi/pdf/10.1111/rssb.12445},
    year = {2021}
}

@InProceedings{papadopoulos2002,
    author={Papadopoulos, Harris and Proedrou, Kostas and Vovk, Volodya and Gammerman, Alex},
    title={Inductive Confidence Machines for Regression},
    booktitle={European Conference on Machine Learning},
    year={2002},
    publisher={Springer Berlin Heidelberg},
    pages={345--356}
}

@article{dumbgen1998,
    author = {Lutz D{\"u}mbgen},
    title = {New goodness-of-fit tests and their application to nonparametric confidence sets},
    volume = {26},
    journal = {The Annals of Statistics},
    number = {1},
    publisher = {Institute of Mathematical Statistics},
    pages = {288 -- 314},
    year = {1998},
    doi = {10.1214/aos/1030563987},
    URL = {https://doi.org/10.1214/aos/1030563987}
}

@article{skorski2023,
    author = {Maciej Skorski},
    title = {Bernstein-type bounds for beta distribution},
    journal = {Modern Stochastics: Theory and Applications},
    volume = {10},
    number = {2},
    year = {2023},
    pages = {211--228},
    doi = {10.15559/23-VMSTA223},
    issn = {2351-6046},
    publisher = {VTeX: Solutions for Science Publishing}
}

@inproceedings{romano2019,
     author = {Romano, Yaniv and Patterson, Evan and Cand{\'e}s, Emmanuel},
     booktitle = {Advances in Neural Information Processing Systems},
     editor = {H. Wallach and H. Larochelle and A. Beygelzimer and F. d\textquotesingle Alch{\'e}-Buc and E. Fox and R. Garnett},
     pages = {},
     publisher = {Curran Associates, Inc.},
     title = {Conformalized Quantile Regression},
     url = {https://proceedings.neurips.cc/paper_files/paper/2019/file/5103c3584b063c431bd1268e9b5e76fb-Paper.pdf},
     volume = {32},
     year = {2019}
}

@article{hu2012,
    title = {Online learning for quantile regression and support vector regression},
    journal = {Journal of Statistical Planning and Inference},
    volume = {142},
    number = {12},
    pages = {3107-3122},
    year = {2012},
    issn = {0378-3758},
    doi = {https://doi.org/10.1016/j.jspi.2012.06.010},
    url = {https://www.sciencedirect.com/science/article/pii/S037837581200211X},
    author = {Ting Hu and Dao-Hong Xiang and Ding-Xuan Zhou}
}

@misc{dua2019,
    author = {Dua, Dheeru and Graff, Casey},
    year = {2017},
    title = {UCI Machine Learning Repository},
    url = {http://archive.ics.uci.edu/ml},
    institution = {University of California, Irvine, School of Information and Computer Sciences}}

@misc{avelin2023,
      title={Sequential inductive prediction intervals}, 
      author={Benny Avelin},
      year={2023},
      eprint={2312.04950},
      archivePrefix={arXiv},
      primaryClass={stat.ME},
      url={https://arxiv.org/abs/2312.04950}, 
}

@article{howard2022,
   title={Sequential estimation of quantiles with applications to A/B testing and best-arm identification},
   volume={28},
   ISSN={1350-7265},
   url={http://dx.doi.org/10.3150/21-BEJ1388},
   DOI={10.3150/21-bej1388},
   number={3},
   journal={Bernoulli},
   publisher={Bernoulli Society for Mathematical Statistics and Probability},
   author={Howard, Steven R. and Ramdas, Aaditya},
   year={2022},
   month=aug 
}

@Misc{howard2021,
  author = {Howard, Steven R. and Waudby-Smith, Ian and Ramdas, Aaditya},
  title = {{ConfSeq}: software for confidence sequences and uniform boundaries},
  year = {2021},
  url = "https://github.com/gostevehoward/confseq",
  note = {[Online; accessed <today>]}
}

@article{gama2014,
    author = {Gama, Jo{\~a}o and \v{Z}liobaitundefined, Indrundefined and Bifet, Albert and Pechenizkiy, Mykola and Bouchachia, Abdelhamid},
    title = {A survey on concept drift adaptation},
    year = {2014},
    issue_date = {April 2014},
    publisher = {Association for Computing Machinery},
    address = {New York, NY, USA},
    volume = {46},
    number = {4},
    issn = {0360-0300},
    url = {https://doi.org/10.1145/2523813},
    doi = {10.1145/2523813},
    journal = {ACM Comput. Surv.},
    month = mar,
    articleno = {44},
    numpages = {37}
}

@article{hoi2021,
    title = {Online learning: A comprehensive survey},
    journal = {Neurocomputing},
    volume = {459},
    pages = {249-289},
    year = {2021},
    issn = {0925-2312},
    doi = {https://doi.org/10.1016/j.neucom.2021.04.112},
    url = {https://www.sciencedirect.com/science/article/pii/S0925231221006706},
    author = {Steven C.H. Hoi and Doyen Sahoo and Jing Lu and Peilin Zhao}
}

@misc{yang2025,
      title={Discounted Online Convex Optimization: Uniform Regret Across a Continuous Interval}, 
      author={Wenhao Yang and Sifan Yang and Lijun Zhang},
      year={2025},
      eprint={2505.19491},
      archivePrefix={arXiv},
      primaryClass={cs.LG},
      url={https://arxiv.org/abs/2505.19491}, 
}

@INPROCEEDINGS {thudi2022,
    author = {Thudi, Anvith and Deza, Gabriel and Chandrasekaran, Varun and Papernot, Nicolas},
    booktitle = { 2022 IEEE 7th European Symposium on Security and Privacy (EuroS\&P) },
    title = {{ Unrolling SGD: Understanding Factors Influencing Machine Unlearning }},
    year = {2022},
    volume = {},
    ISSN = {},
    pages = {303-319},
    doi = {10.1109/EuroSP53844.2022.00027},
    url = {https://doi.ieeecomputersociety.org/10.1109/EuroSP53844.2022.00027},
    publisher = {IEEE Computer Society},
    address = {Los Alamitos, CA, USA},
    month =Jun}

@InProceedings{you2016,
    author={You, Sheng-Chi and Lin, Hsuan-Tien},
    editor={Bailey, James and Khan, Latifur and Washio, Takashi and Dobbie, Gill and Huang, Joshua Zhexue and Wang, Ruili},
    title="A Simple Unlearning Framework for Online Learning Under Concept Drifts",
    booktitle="Advances in Knowledge Discovery and Data Mining",
    year="2016",
    publisher="Springer International Publishing",
    address="Cham",
    pages="115--126",
    isbn="978-3-319-31753-3"
}

@article{lu2019,
  author={Lu, Jie and Liu, Anjin and Dong, Fan and Gu, Feng and Gama, Jo{\~a}o and Zhang, Guangquan},
  journal={IEEE Transactions on Knowledge and Data Engineering}, 
  title={Learning under Concept Drift: A Review}, 
  year={2019},
  volume={31},
  number={12},
  pages={2346-2363},
  doi={10.1109/TKDE.2018.2876857}}

@article{zhang2017,
    title = {Online sequential ELM algorithm with forgetting factor for real applications},
    journal = {Neurocomputing},
    volume = {261},
    pages = {144-152},
    year = {2017},
    note = {Advances in Extreme Learning Machines (ELM 2015)},
    issn = {0925-2312},
    doi = {https://doi.org/10.1016/j.neucom.2016.09.121},
    url = {https://www.sciencedirect.com/science/article/pii/S0925231217302205},
    author = {Haigang Zhang and Sen Zhang and Yixin Yin}
}

@article{liu2016,
    title = {FP-ELM: An online sequential learning algorithm for dealing with concept drift},
    journal = {Neurocomputing},
    volume = {207},
    pages = {322-334},
    year = {2016},
    issn = {0925-2312},
    doi = {https://doi.org/10.1016/j.neucom.2016.04.043},
    url = {https://www.sciencedirect.com/science/article/pii/S0925231216303125},
    author = {Dong Liu and YouXi Wu and He Jiang}
}

@misc{gauthier2025, 
	title={E-Values Expand the Scope of Conformal Prediction}, 
	author={Etienne Gauthier and Francis Bach and Michael I. Jordan}, 
	year={2025}, 
	eprint={2503.13050}, 
	archivePrefix={arXiv}, 
	primaryClass={stat.ML}, 
	url={https://arxiv.org/abs/2503.13050}
 }

\appendix
\section{Proofs of the main results} \label{proofs}

\subsection{Proof of Proposition \ref{prop:equiv_TUC_goal}} \label{equiv_TUC_proof}
The backward implication follows from the fact that \begin{align*}
    \mathbb{P}\left(Z \in \widehat{C}_{T, \alpha}\right) &= \mathbb{E}\left\{\mathbb{E}\left[\mathbbm{1}\{Z \in \widehat{C}_{T, \alpha}\} \vert Z_1, \dots, Z_T\right]\right\} \\
    &= \mathbb{E}\left\{\mathbb{P}\left(\mathbbm{1}\{Z \in \widehat{C}_{T, \alpha}\} \vert Z_1, \dots, Z_T\right)\right\} \\
    &= \mathbb{E}\left[\mu_Z\left(\widehat{C}_{T, \alpha}\right)\right] \\
    &\geq \mathbb{E}\left[\min_{t \geq 1} \mu_Z\left(\widehat{C}_{t, \alpha}\right)\right].
\end{align*} This immediately implies that \begin{equation*}
    \mathbb{P}\left(Z \in \widehat{C}_{T, \alpha}\right) \geq 1-\alpha \text{ for all random times $T$ } \impliedby \mathbb{E}\left[\min_{t \geq 1} \mu_Z\left(\widehat{C}_{t, \alpha}\right)\right] \geq 1-\alpha.
\end{equation*} The forwards implication follows from considering the random time $T = \argmin_{t\geq 1}\mu_Z\left(\widehat{C}_{t, \alpha}\right)$. Then, \begin{align*}
    \mathbb{P}\left(Z \in \widehat{C}_{T, \alpha}\right) \geq 1 - \alpha \text{ for all random times $T$ } &\implies \mathbb{P}\left(Z \in \widehat{C}_{\argmin_{t\geq 1}\mu_Z\left(\widehat{C}_{t, \alpha}\right), \alpha}\right) \geq 1 - \alpha \\
    &\iff \mathbb{E}\left[\mathbb{P}\left(Z \in \widehat{C}_{\argmin_{t\geq 1}\mu_Z\left(\widehat{C}_{t, \alpha}\right), \alpha} \ \vert \ \left\{Z_t\right\}_{t=1}^{\infty}\right)\right] \geq 1-\alpha \\
    &\iff \mathbb{E}\left[\mu_Z\left(\widehat{C}_{\argmin_{t\geq 1}\mu_Z\left(\widehat{C}_{t, \alpha}\right), \alpha}\right)\right] \geq 1-\alpha \\
    &\iff \mathbb{E}\left[\min_{t \geq 1} \mu_Z\left(\widehat{C}_{t, \alpha}\right)\right] \geq 1-\alpha.
\end{align*}

\subsection{Proof of Proposition \ref{prop:equiv_TUPAC_goal}}
Analogously to the TUC goal, the TUPAC guarantee can be written in 2 equivalent forms.
\begin{proposition} \label{prop:equiv_TUPAC_goal}
Suppose $Z_t, Z \sim \mu_Z$ for each $t = 1, 2, \dots$. Then, \small\begin{equation}        
    \mathbb{P}\left\{\mathbb{P}\left(Z \in \widehat{C}_{T, \alpha, \delta} \vert Z_1, \dots, Z_T\right) \geq 1-\alpha\right\} \geq 1 - \delta \text{ for all random $T$ } \iff \mathbb{P}\left\{\min_{t \geq 1} \mu_Z(\widehat{C}_{t, \alpha, \delta}) \geq 1 - \alpha\right\} \geq 1 - \delta.
\end{equation}
\end{proposition}
The backwards implication follows from the fact that \begin{equation*}
    \mathbb{P}\left\{\mathbb{P}\left(Z \in \widehat{C}_{T, \alpha, \delta} \vert Z_1, \dots, Z_T\right) \geq 1-\alpha\right\} = \mathbb{P}\left(\mu_Z\left(\widehat{C}_{T, \alpha, \delta}\right) \geq 1-\alpha\right) \geq \mathbb{P}\left(\min_{t \geq 1} \mu_Z\left(\widehat{C}_{t, \alpha, \delta}\right) \geq 1-\alpha\right).
\end{equation*} This immediately implies that {\small\begin{equation*}
    \mathbb{P}\left\{\mathbb{P}\left(Z \in \widehat{C}_{T, \alpha, \delta} \vert Z_1, \dots, Z_T\right) \geq 1-\alpha\right\} \geq 1 - \delta \text{ for all random $T$ } \impliedby \mathbb{P}\left(\min_{t \geq 1} \mu_Z\left(\widehat{C}_{t, \alpha, \delta}\right) \geq 1-\alpha\right) \geq 1 - \delta.
\end{equation*}} The forwards implication follows from considering the random time $T = \argmin_{t \geq 1} \mu_Z\left(\widehat{C}_{t, \alpha, \delta}\right)$. Then, \begin{align*}
    &\mathbb{P}\left\{\mathbb{P}\left(Z \in \widehat{C}_{T, \alpha, \delta} \vert Z_1, \dots, Z_T\right) \geq 1-\alpha\right\} \geq 1 - \delta \text{ for all random $T$ } \implies \\
    &\qquad \mathbb{P}\left\{\mathbb{P}\left(Z \in \widehat{C}_{\argmin_{t \geq 1} \mu_Z\left(\widehat{C}_{t, \alpha, \delta}\right), \alpha, \delta} \vert \{Z_t\}_{t=1}^{\infty}\right) \geq 1-\alpha\right\} \geq 1 - \delta \iff \\
    &\qquad \mathbb{P}\left\{\mu_Z\left(\widehat{C}_{\argmin_{t \geq 1} \mu_Z\left(\widehat{C}_{t, \alpha, \delta}\right), \alpha, \delta}\right) \geq 1-\alpha\right\} \geq 1 - \delta \iff \\
    &\qquad \mathbb{P}\left\{\min_{t \geq 1} \mu_Z\left(\widehat{C}_{t, \alpha, \delta}\right) \geq 1-\alpha\right\} \geq 1 - \delta.
\end{align*}

\subsection{Proof of Theorem \ref{fixed_confidence_sequence_pac_thm}} \label{fixed_confidence_sequence_pac_proof}
From Proposition \ref{prop:equiv_TUPAC_goal} and the definition of $\mu_Z$, the TUPAC goal can be restated as \begin{equation*}
    \mathbb{P}\left\{\forall t \in \mathbb{N}: \mathbb{P}\left(Z \in \widehat{C}_{t, \alpha, \delta} \ \vert \ Z_1, \dots, Z_t\right) \geq 1-\alpha\right\} \geq 1-\delta.
\end{equation*} By the definition of $\widehat{C}_{t, \alpha, \delta}$ in Algorithm \ref{fixed_anytime_alg}, the preceding goal becomes \begin{equation*}
    \mathbb{P}\left\{\forall t \in \mathbb{N}: \mathbb{P}\left(R(Z) \leq \widehat{q}_{t, \alpha, \delta} \ \vert \ Z_1, \dots, Z_t\right) \geq 1-\alpha\right\} \geq 1-\delta.
\end{equation*} Letting $F_R$ be the CDF of $R(Z)$ and noting that $Z$ is independent of $Z_1, \dots, Z_t$, we see that the goal is equivalent to \begin{equation*}
    \mathbb{P}\left\{\forall t \in \mathbb{N}: F_R\left(\widehat{q}_{t, \alpha, \delta}\right) \geq 1-\alpha\right\} \geq 1-\delta.
\end{equation*} Now, Theorem 1 of \cite{howard2022} tells us that \begin{equation*}
    \mathbb{P}\left\{\forall t \in \mathbb{N}: \inf\left\{z: F_R\left(z\right) \geq 1-\alpha\right\} \leq \widehat{q}_{t, \alpha, \delta}\right\} \geq 1-\delta
\end{equation*} which immediately implies the former and so we conclude the proof.

\subsection{Proof of Theorem \ref{fixed_anytime_conformal_thm}} \label{fixed_anytime_conformal_proof}
The result relies on the following lemma about uniform order statistics. 
\begin{lemma} \label{fixed_anytime_conformal_lemma}
Let $U_1, \dots, U_t$ be independent standard uniform random variables and take $U_{1:t} \leq \dots \leq U_{t:t}$ to be the order statistics of $\{U_s\}_{s=1}^t$. Let $h : \mathbb{Z}^{\geq 0} \to \mathbb{R}$ satisfy $\sum_{t=0}^{\infty} h(t) = 1$. Then, for any $\alpha \in (0.5, 1]$, \begin{equation*}
    \mathbb{E}\left[\min_{t > t_0} U_{\lceil (t+1)(1-\alpha + u_t)\rceil:t}\right] \geq 1-\alpha
\end{equation*} where \begin{equation*}
    u_t = \frac{4(2\alpha - 1)\log\left(\frac{1}{h(t)}\right)}{3(t+3)} + \sqrt{\frac{2\alpha(1-\alpha)\log\left(\frac{1}{h(t)}\right)}{t+2}} + \frac{1}{2}\sqrt{\frac{2\pi\alpha(1-\alpha)}{t+2}}\left(1 - \sum_{s=0}^{t_0}h(s)\right),
\end{equation*} and $t_0$ is the smallest positive integer satisfying \begin{equation*}
    (t+1)(1-\alpha) \leq (t+1)\left(1-\alpha + u_t\right) \leq t \ \forall t > t_0.
\end{equation*}
\end{lemma}
Taking this lemma as given, the result follows from writing the desired guarantee in terms of uniform order statistics. By Proposition \ref{prop:equiv_TUC_goal}, it suffices to show that \begin{equation*}
    \mathbb{E}\left[\min_{t > 0} \mu_Z\left(\widehat{C}_{t, \alpha}\right)\right] \geq 1 - \alpha.
\end{equation*} Letting $\widehat{R}_{1:t} \leq \dots \leq \widehat{R}_{t:t}$ be the order statistics of $\{R(Z_s)\}_{s=1}^t$, define \begin{equation*}
    \widehat{q}_{t, \alpha} = \begin{cases}\widehat{R}_{\lceil(t+1)(1-\alpha + u_t)\rceil:t} & \text{if } t > t_0 \\ \infty & \text{otherwise}\end{cases} 
\end{equation*} so that $\widehat{C}_{t, \alpha} = \left\{z : R(z) \leq \widehat{q}_{t, \alpha}\right\} = R^{-1}\left(-\infty, \widehat{R}_{\lceil(t+1)(1-\alpha + u_t)\rceil:t}\right]$ if $t > t_0$. Now, notice that \begin{align*}
    \mathbb{E}\left[\min_{t > 0} \mu_Z\left(\widehat{C}_t\right)\right] &= \mathbb{E}\left[\min_{t > t_0} \mu_Z\left(\widehat{C}_t\right)\right] \\
    &= \mathbb{E}\left[\min_{t > t_0} \mu_Z\left(R^{-1}\left(\infty, \widehat{R}_{\lceil(t+1)(1-\alpha + u_t)\rceil:t}\right]\right)\right] \\
    &= \mathbb{E}\left[\min_{t > t_0} F_R\left(\widehat{R}_{\lceil(t+1)(1-\alpha + u_t)\rceil:t}\right)\right]
\end{align*} where $F_R$ is the CDF of $R(Z)$. Furthermore, letting $\{U_s\}_{s=1}^t$ be independent standard uniform random variables with order statistics $U_{1:t} \leq \dots \leq U_{t:t}$, properties of the CDF imply that $F_R\left(\widehat{R}_{\lceil(t+1)(1-\alpha + u_t)\rceil:t}\right) \overset{d}{=} U_{\lceil(t+1)(1-\alpha + u_t)\rceil:t}$. Thus, \begin{equation*}
    \mathbb{E}\left[\min_{t > t_0} U_{\lceil(t+1)(1-\alpha + u_t)\rceil:t} \right] \geq 1 - \alpha \implies \eqref{eq:tuc_goal}
\end{equation*} and applying Lemma \ref{fixed_anytime_conformal_lemma} yields the desired guarantee. 

\begin{proof}[Lemma \ref{fixed_anytime_conformal_lemma}]
It is well known that $U_{\lceil(t+1)(1-\alpha + u_t)\rceil: t} \sim \text{Beta}(\lceil(t+1)(1-\alpha + u_t)\rceil, t+1 - \lceil(t+1)(1-\alpha + u_t)\rceil)$. Then, we can apply Theorem 1 from \cite{skorski2023} to conclude that if $\alpha < \frac{1}{2}$, \begin{equation*}
    \mathbb{P}\left(\frac{\lceil(t+1)(1-\alpha + u_t)\rceil}{t+1} - U_{\lceil(t+1)(1-\alpha + u_t)\rceil:t} > \epsilon\right) \leq \exp\left(-\frac{\epsilon^2}{2\left(v_t + \frac{c_t\epsilon}{3}\right)}\right) \ \forall \epsilon > 0, t > t_0
\end{equation*} where $v_t = \frac{\frac{\lceil(t+1)(1-\alpha + u_t)\rceil}{t+1}\left(1 - \frac{\lceil(t+1)(1-\alpha + u_t)\rceil}{t+1}\right)}{t+2}$ and $c_t = \frac{2\left(1-2\frac{\lceil(t+1)(1-\alpha + u_t)\rceil}{t+1}\right)}{t+3}$. This follows from choosing $t_0$ so that $(t+1)(1-\alpha) \leq \lceil(t+1)(1-\alpha + u_t)\rceil \leq t \ \forall t > t_0$ which implies that $t+1 -\lceil(t+1)(1-\alpha + u_t)\rceil \leq (1-\alpha)(t+1) \leq \lceil(t+1)(1-\alpha + u_t)\rceil$ when $\alpha < \frac{1}{2}$. Next, let $\delta^2 - \log h(t) = \frac{\epsilon^2}{2\left(v_t + \frac{c_t\epsilon}{3}\right)}$. Equivalently, $\epsilon$ solves the quadratic equation \begin{equation*}
    \epsilon^2 - \frac{2c_t}{3}(\delta^2 - \log h(t))\epsilon - 2v_t(\delta^2 - \log h(t)) = 0,
\end{equation*} so \begin{equation*}
    \epsilon = \frac{1}{2}\left(\frac{2c_t}{3}(\delta^2 - \log h(t)) + \sqrt{\left(\frac{2c_t}{3}(\delta^2 - \log h(t))\right)^2 + 8v_t(\delta^2 - \log h(t))}\right)
\end{equation*} is the positive solution. It follows that \begin{align*}
    &\mathbb{P}\Biggl(\frac{\lceil(t+1)(1-\alpha + u_t)\rceil}{t+1} - U_{\lceil(t+1)(1-\alpha + u_t)\rceil:t} > \\
    &\qquad\frac{1}{2}\left(\frac{2c_t}{3}(\delta^2 -\log h(t)) + \sqrt{\left(\frac{2c_t}{3}(\delta^2 - \log h(t))\right)^2 + 8v_t(\delta^2 - \log h(t))}\right)\Biggr) \leq e^{-\delta^2}h(t)
\end{align*} for all $\delta > 0$ and $t > t_0$. Now, since $\sqrt{a + b} \leq \sqrt{a} + \sqrt{b}$, we have that \begin{align*}
    &\mathbb{P}\Biggl(\frac{\lceil(t+1)(1-\alpha + u_t)\rceil}{t+1} - U_{\lceil(t+1)(1-\alpha + u_t)\rceil:t} > \\
    &\qquad\frac{2c_t}{3}(\delta^2 - \log h(t)) + \sqrt{2v_t}\delta + \sqrt{-2v_t\log h(t)}\Biggr) \leq e^{-\delta^2}h(t)
\end{align*} for all $\delta > 0$ and $t > t_0$. Notice that for all $t > t_0$, $v_t \leq \frac{\alpha(1-\alpha)}{t+2}$ and $c_t \leq \frac{2(2\alpha-1)}{t+3}$. Therefore, it follows that \begin{align*}
    &\mathbb{P}\Biggl(\frac{\lceil(t+1)(1-\alpha + u_t)\rceil}{t+1} - U_{\lceil(t+1)(1-\alpha + u_t)\rceil:t} > \\
    &\qquad\frac{4(2\alpha - 1)}{3(t+3)}(\delta^2 - \log h(t)) + \sqrt{\frac{2\alpha(1-\alpha)}{t+2}}\delta + \sqrt{\frac{2\alpha(1-\alpha)\log\left(\frac{1}{h(t)}\right)}{t+2}}\Biggr) \leq e^{-\delta^2}h(t)
\end{align*} for all $\delta > 0$ and $t > t_0$. Equivalently, we have that \begin{align*}
    &\mathbb{P}\Biggl(\sqrt{\frac{t+2}{2\alpha(1-\alpha)}}\Biggl\{\frac{\lceil(t+1)(1-\alpha + u_t)\rceil}{t+1} - U_{\lceil(t+1)(1-\alpha + u_t)\rceil:t} + \frac{4(2\alpha - 1)\log h(t)}{3(t+3)} -\\
    &\phantom{\mathbb{P}\Biggl(\sqrt{\frac{t+2}{2\alpha(1-\alpha)}}\Biggl\{}\ \sqrt{\frac{2\alpha(1-\alpha)\log\left(\frac{1}{h(t)}\right)}{t+2}}\Biggr\} > \frac{4(2\alpha - 1)\sqrt{t+2}}{3(t+3)\sqrt{2\alpha(1-\alpha)}}\delta^2 + \delta\Biggr) \leq e^{-\delta^2}h(t)
\end{align*} for all $\delta > 0$ and $t > t_0$. Since $\frac{4(2\alpha - 1)\sqrt{t+2}}{3(t+3)\sqrt{2\alpha(1-\alpha)}}\delta^2 \leq 0$, it follows that \begin{align*}
    &\mathbb{P}\Biggl(\sqrt{\frac{t+2}{2\alpha(1-\alpha)}}\Biggl\{\frac{\lceil(t+1)(1-\alpha + u_t)\rceil}{t+1} - U_{\lceil(t+1)(1-\alpha + u_t)\rceil:t} + \frac{4(2\alpha - 1)\log h(t)}{3(t+3)} - \\
    &\phantom{\mathbb{P}\Biggl(\sqrt{\frac{t+2}{2\alpha(1-\alpha)}}\Biggl\{}\ \sqrt{\frac{2\alpha(1-\alpha)\log \left(\frac{1}{h(t)}\right)}{t+2}}\Biggr\} > \delta\Biggr) \leq e^{-\delta^2}h(t)
\end{align*} for all $\delta > 0$ and $t > t_0$. From here, applying a union bound argument yields \begin{align*}
    &\mathbb{P}\Biggl(\max_{t > t_0} \sqrt{\frac{t+2}{2\alpha(1-\alpha)}}\Biggl\{\frac{\lceil(t+1)(1-\alpha + u_t)\rceil}{t+1} - U_{\lceil(t+1)(1-\alpha + u_t)\rceil:t} + \frac{4(2\alpha - 1)\log h(t)}{3(t+3)} - \\
    &\phantom{\mathbb{P}\Biggl(\max_{t > t_0} \sqrt{\frac{t+2}{2\alpha(1-\alpha)}}\Biggl\{}\ \sqrt{\frac{2\alpha(1-\alpha)\log \left(\frac{1}{h(t)}\right)}{t+2}}\Biggr\} > \delta\Biggr) \leq e^{-\delta^2}\sum_{t=t_0+1}^{\infty}h(t) = e^{-\delta^2}\left(1 - \sum_{t=0}^{t_0}h(t)\right)
\end{align*} for all $\delta > 0$. Now, \begin{align*}
    &\mathbb{E}\Biggl[\max_{t > t_0} \sqrt{\frac{t+2}{2\alpha(1-\alpha)}}\Biggl\{\frac{\lceil(t+1)(1-\alpha + u_t)\rceil}{t+1} - U_{\lceil(t+1)(1-\alpha + u_t)\rceil:t} + \frac{4(2\alpha - 1)\log h(t)}{3(t+3)} - \\
    &\phantom{\mathbb{E}\Biggl[\max_{t > t_0} \sqrt{\frac{t+2}{2\alpha(1-\alpha)}}\Biggl\{}\ \sqrt{\frac{2\alpha(1-\alpha)\log\left(\frac{1}{h(t)}\right)}{t+2}}\Biggr\}\Biggr] \leq \frac{\sqrt{\pi}}{2}\left(1 - \sum_{t=0}^{t_0}h(t)\right).
\end{align*} Rearranging, it follows that \begin{equation*}
    \mathbb{E}\left[\max_{t > t_0} \left\{\frac{\lceil(t+1)(1-\alpha + u_t)\rceil}{t+1} - U_{\lceil(t+1)(1-\alpha + u_t)\rceil:t} - u_t\right\}\right] \leq 0.
\end{equation*} So, \begin{align*}
    \mathbb{E}\left[\min_{t > t_0} U_{\lceil(t+1)(1-\alpha + u_t)\rceil:t}\right] &= \mathbb{E}\Biggl[-\max_{t > t_0} \Biggl\{\frac{\lceil(t+1)(1-\alpha + u_t)\rceil}{t+1} - U_{\lceil(t+1)(1-\alpha + u_t)\rceil:t} - u_t - \\
    &\phantom{=\mathbb{E}\Biggl[-\max_{t > t_0} \Biggl\{}\ \frac{\lceil(t+1)(1-\alpha + u_t)\rceil}{t+1} + u_t\Biggr\}\Biggr] \\
    &\geq \min_{t > t_0} \left\{\frac{\lceil(t+1)(1-\alpha + u_t)\rceil}{t+1} - u_t\right\} \\
    &\geq 1-\alpha.
\end{align*}
\end{proof}

\subsection{Proof of Theorem \ref{fixed_anytime_pac_thm}} 
\label{fixed_anytime_pac_proof}

The result relies on the following lemma about uniform order statistics. 
\begin{lemma} \label{fixed_anytime_pac_lemma}
Let $U_1, \dots, U_t$ be independent standard uniform random variables and take $U_{1:t} \leq \dots \leq U_{t:t}$ to be the order statistics of $\{U_s\}_{s=1}^t$. Let $h : \mathbb{Z}^{\geq 0} \to \mathbb{R}$ satisfy $\sum_{t=0}^{\infty} h(t) = 1$. Then, for any $\alpha, \delta \in (0, 1)$, \begin{equation*}
    \mathbb{P}\left(\min_{t > t_0} U_{\beta:t} \geq 1 - \alpha\right) \geq 1-\delta
\end{equation*} where $\beta$ is the smallest natural number greater than $(1-\alpha)(t+1)$ such that \begin{equation*}
    \psi\left(1-\alpha, \frac{\beta}{t+1}\right) \geq u_t := \frac{\log\left(\frac{1}{\delta}\left(1-\sum_{s=0}^{t_0} h(s)\right)\right) - \log h(t)}{t+1},
\end{equation*} and $t_0$ is the smallest natural number such that $\beta \leq t$ for all $t > t_0$.
\end{lemma}
Taking this lemma as given, the result follows from writing the desired guarantee in terms of uniform order statistics. By Proposition \ref{prop:equiv_TUPAC_goal}, it suffices to show that \begin{equation*}
    \mathbb{P}\left(\min_{t \geq 1} \mu_Z\left(\widehat{C}_{t, \alpha, \delta}\right) \geq 1-\alpha\right) \geq 1-\delta.
\end{equation*} Letting $\widehat{R}_{1:t} \leq \dots \leq \widehat{R}_{t:t}$ be the order statistics of $\left\{R(Z_s)\right\}_{s=1}^t$, we can write $\widehat{q}_{t, \alpha, \delta}$ from Theorem \ref{fixed_anytime_pac_thm} as \begin{equation*}
    \widehat{q}_{t, \alpha, \delta} = \begin{cases}\widehat{R}_{\beta:t} & \text{if } t > t_0 \\ \infty & \text{otherwise}\end{cases}
\end{equation*} so that $\widehat{C}_{t, \alpha, \delta} = \left\{z: R(z) \leq \widehat{q}_{t, \alpha, \delta}\right\} = R^{-1}\left(-\infty, \widehat{R}_{\beta:t}\right]$ if $t > t_0$. Now, notice that \begin{align*}
    \mathbb{P}\left(\min_{t \geq 1} \mu_Z\left(\widehat{C}_{t, \alpha, \delta}\right) \geq 1-\alpha\right) &= \mathbb{P}\left(\min_{t > t_0} \mu_Z\left(\widehat{C}_{t, \alpha, \delta}\right) \geq 1-\alpha\right) \\
    &= \mathbb{P}\left(\min_{t > t_0} \mu_Z\left(R^{-1}\left(-\infty, \widehat{R}_{\beta:t}\right]\right) \geq 1-\alpha\right) \\
    &= \mathbb{P}\left(\min_{t > t_0} F_R\left(\widehat{R}_{\beta:t}\right) \geq 1-\alpha\right)
\end{align*} where $F_R$ is the CDF of $R(Z)$. Furthermore, letting $\{U_s\}_{s=1}^t$ be independent standard uniform random variables with order statistics $U_{1:t} \leq \dots \leq U_{t:t}$, properties of the CDF imply that $F_R\left(\widehat{R}_{\beta:t}\right) \overset{d}{=} U_{\beta:t}$. Thus, \begin{equation*}
     \mathbb{P}\left(\min_{t > t_0} U_{\beta:t} \geq 1-\alpha\right) \geq 1 - \delta \implies \eqref{eq:anytime_pac_goal}
\end{equation*} and applying Lemma \ref{fixed_anytime_pac_lemma} yields the desired guarantee. 

\begin{proof}[Lemma \ref{fixed_anytime_pac_lemma}]
It is well known that $U_{\beta:t} \sim \text{Beta}(\beta, t - \beta +1)$. Therefore, Proposition 2.1 of \cite{dumbgen1998} implies that \begin{equation} \label{eq:dumbgen1998} \begin{split}
    \mathbb{P}\left(U_{\beta:t} \geq x\right) &\leq e^{-(t+1)\psi\left(x, \frac{\beta}{t+1}\right)} \text{ if } x \geq \frac{\beta}{t+1} \text{ and} \\
    \mathbb{P}\left(U_{\beta:t} \leq x\right) &\leq e^{-(t+1)\psi\left(x, \frac{\beta}{t+1}\right)} \text{ if } x \leq \frac{\beta}{t+1}
\end{split}\end{equation} where \begin{equation*}
    \psi(x, p) := p\log\left(\frac{p}{x}\right) + (1-p)\log\left(\frac{1-p}{1-x}\right).
\end{equation*} Now, using a union bound argument, we have that if $\beta \geq (t+1)(1-\alpha)$, \begin{align*}
    \mathbb{P}\left(\min_{t > t_0} U_{\beta:t} \geq 1-\alpha\right) &= 1 - \mathbb{P}\left(\min_{t > t_0} U_{\beta:t} < 1-\alpha\right) \\
    &= 1 - \mathbb{P}\left(\exists \ t > t_0 \text{ s.t. } U_{\beta:t} < 1-\alpha\right) \\
    &\geq 1 - \sum_{t > t_0} \mathbb{P}\left(U_{\beta:t} < 1-\alpha\right) \\
    &\geq 1 - \sum_{t > t_0} e^{-(t+1)\psi\left(1-\alpha, \frac{\beta}{t+1}\right)}.
\end{align*} The claim follows immediately if we choose $\beta$ such that $\psi\left(1-\alpha, \frac{\beta}{t+1}\right) \geq u_t$. 
\end{proof}

\subsection{Proof of Theorem \ref{dynamic_anytime_conformal_thm}} \label{dynamic_anytime_conformal_proof}
The result relies on the following lemma about uniform order statistics. 
\begin{lemma} \label{dynamic_anytime_conformal_lemma}
Let $s_{t, k} := t - \left(\lceil \eta^k \rceil - 1\right)$ for some $\eta > 0$. Suppose $U_1, \dots, U_{s_{t, k}}$ are independent standard uniform random variables and take $U_{1:s_{t, k}} \leq \dots \leq U_{s_{t,k}:s_{t, k}}$ to be the order statistics of $\{U_s\}_{s=1}^{s_{t, k}}$. Let $h, g : \mathbb{Z}^{\geq 0} \to \mathbb{R}$ satisfy $\sum_{t=0}^{\infty} h(t) = \sum_{t=0}^{\infty} g(t) = 1$. Then, for any $\alpha \in (0.5, 1]$, \begin{equation*}
    \mathbb{E}\left[\min_{k \geq 0} \min_{t_k < t < \eta^{k+2}} U_{\lceil (s_{t, k} + 1)(1-\alpha + u_{t, k})\rceil:s_{t, k}}\right] \geq 1-\alpha
\end{equation*} where \begin{align*}
    u_{t,k} &= \frac{4(2\alpha - 1)\log \left(\frac{1}{h(s_{t, k})}\right)}{3(s_{t, k}+3)} + \sqrt{\frac{2\alpha(1-\alpha)}{s_{t, k}+2}}\left(\sqrt{\log \left(\frac{1}{h(s_{t, k})}\right)} + \sqrt{\log\left(\sum_{t'=t_k+1}^{\lfloor \eta^{k+2}\rfloor}\frac{h(s_{t', k})}{g(k)}\right)} + \frac{\sqrt{\pi}}{2}\right)
\end{align*} with $t_k > \eta^k$ being the smallest real number satisfying \begin{align*}
    (1 - \alpha)(s_{t, k} + 1) \leq (s_{t, k} + 1)(1-\alpha + u_{t, k}) \leq s_{t, k} \ \forall t \in (t_k, \eta^{k+2}).
\end{align*}
\end{lemma}
Taking this lemma as given, the result follows from writing the desired guarantee in terms of uniform order statistics. As we saw in the fixed transformation setting, it suffices to show that \begin{equation} \label{eq:from_fixed}
    \mathbb{E}\left[\min_{t>0} \mu_Z(\widehat{C}_t)\right] \geq 1 - \alpha.
\end{equation} Let $\widehat{R}^{(k)}_{1:s_{t, k}} \leq \dots \widehat{R}^{(k)}_{s_{t, k}:s_{t, k}}$ be the order statistics of $\left\{\widehat{R}^{(k)}(Z_s)\right\}_{s=\lceil\eta^k\rceil}^t$ and define \begin{equation*}
    \widehat{q}_{t, k} = \begin{cases}
        \widehat{R}^{(k)}_{\lceil (s_{t, k} + 1)(1-\alpha + u_{t, k})\rceil:s_{t, k}} & \text{if } t > t_k \geq \eta^k \\
        \infty & \text{otherwise}.
    \end{cases}
\end{equation*} In Algorithm \ref{alg:dynamic_algo-illustration}, the transformation $\widehat{R}^{(k)}$ is used to build prediction sets for $\eta^{k} \leq t < \eta^{k+2}$. Thus, \eqref{eq:from_fixed} becomes \begin{equation*}
    \mathbb{E}\left[\min_{k \geq 0} \min_{t_k < t < \eta^{k+2}} F_k\left((\widehat{R}^{(k)})^{-1}\left(-\infty, \widehat{R}^{(k)}_{\lceil (s_{t, k} + 1)(1-\alpha + u_{t, k})\rceil:s_{t, k}}\right]\right)\right] \geq 1 - \alpha
\end{equation*} where $F_k$ is the CDF of $\widehat{R}^{(k)}(Z)$. Since $F_k\left((\widehat{R}^{(k)})^{-1}\left(-\infty, \widehat{R}^{(k)}_{\lceil (s_{t, k} + 1)(1-\alpha + u_{t, k})\rceil:s_{t, k}}\right]\right) \overset{d}{=} U_{\lceil (s_{t, k} + 1)(1-\alpha + u_{t, k})\rceil:s_{t, k}}$ where $U_{1:s_{t, k}} \leq \dots \leq U_{s_{t, k}:s_{t, k}}$ are the order statistics of $s_{t, k}$ standard uniform random variables $\{U_s\}_{s=1}^{s_{t, k}}$, the previous inequality becomes \begin{equation*}
    \mathbb{E}\left[\min_{k \geq 0} \min_{t_k < t < \eta^{k+2}} U_{\lceil (s_{t, k} + 1)(1-\alpha + u_{t, k})\rceil:s_{t, k}}\right] \geq 1 - \alpha.
\end{equation*} Thus, \begin{equation*}
    \mathbb{E}\left[\min_{k \geq 0} \min_{t_k < t < \eta^{k+2}} U_{\lceil (s_{t, k} + 1)(1-\alpha + u_{t, k})\rceil:s_{t, k}}\right] \geq 1 - \alpha \implies \eqref{eq:tuc_goal}
\end{equation*} and we conclude the proof by applying Lemma \ref{dynamic_anytime_conformal_lemma}. 

\begin{proof}[Lemma \ref{dynamic_anytime_conformal_lemma}]
For each $k \geq 0$, we choose $t_k > \eta^k$ so that $(s_{t, k} + 1)(1-\alpha) \leq (s_{t, k} + 1)(1-\alpha + u_{t, k}) \leq s_{t, k}$ for all $t > t_k$. Assuming $\alpha < \frac{1}{2}$, we can follow an analogous proof to Lemma \ref{fixed_anytime_conformal_lemma} to obtain \begin{align*}
    &\mathbb{P}\Biggl(\sqrt{\frac{s_{t, k}+2}{2\alpha(1-\alpha)}}\Biggl\{\frac{\lceil (s_{t, k} + 1)(1-\alpha + u_{t, k})\rceil}{s_{t, k} + 1} - U_{\lceil (s_{t, k} + 1)(1-\alpha + u_{t, k})\rceil:s_{t, k}} + \frac{4(2\alpha - 1)\log h(s_{t, k})}{3(s_{t, k}+3)} -\\
    &\phantom{\mathbb{P}\Biggl(\sqrt{\frac{s_{t, k}+2}{2\alpha(1-\alpha)}}\Biggl\{}\ \sqrt{\frac{2\alpha(1-\alpha)\log \left(\frac{1}{h(s_{t, k})}\right)}{s_{t, k}+2}}\Biggr\} > \delta\Biggr) \leq e^{-\delta^2}h(s_{t, k})
\end{align*} for all $\delta > 0$, $t > t_k$, $k \geq 0$. Applying a union bound argument further yields \begin{align*}
   &\mathbb{P}\Biggl(\max_{t_k < t < \eta^{k+2}} \sqrt{\frac{s_{t, k}+2}{2\alpha(1-\alpha)}}\Biggl\{\frac{\lceil (s_{t, k} + 1)(1-\alpha + u_{t, k})\rceil}{s_{t, k} + 1} - U_{\lceil (s_{t, k} + 1)(1-\alpha + u_{t, k})\rceil:s_{t, k}} + \frac{4(2\alpha - 1)\log h(s_{t, k})}{3(s_{t, k}+3)} - \\
   &\phantom{\mathbb{P}\Biggl(\max_{t_k < t < \eta^{k+2}} \sqrt{\frac{s_{t, k}+2}{2\alpha(1-\alpha)}}\Biggl\{}\ \sqrt{\frac{2\alpha(1-\alpha)\log \left(\frac{1}{h(s_{t, k})}\right)}{s_{t, k}+2}}\Biggr\} > \delta\Biggr) \leq e^{-\delta^2}\sum_{t=t_k+1}^{\lfloor \eta^{k+2}\rfloor}h(s_{t, k})
\end{align*} for all $\delta > 0$, $k \geq 0$. Now, set $\gamma g(k) = e^{-\delta^2}\sum_{t=t_k+1}^{\lfloor \eta^{k+2}\rfloor}h(s_{t, k})$ so that $\delta = \sqrt{\log\left(\frac{\sum_{t=t_k+1}^{\lfloor \eta^{k+2}\rfloor}h(s_{t, k})}{\gamma g(k)}\right)}$. Then, \begin{align*}
    &\mathbb{P}\Biggl(\max_{t_k < t < \eta^{k+2}} \sqrt{\frac{s_{t, k}+2}{2\alpha(1-\alpha)}}\Biggl\{\frac{\lceil (s_{t, k} + 1)(1-\alpha + u_{t, k})\rceil}{s_{t, k} + 1} - U_{\lceil (s_{t, k} + 1)(1-\alpha + u_{t, k})\rceil:s_{t, k}} + \frac{4(2\alpha - 1)\log h(s_{t, k})}{3(s_{t, k}+3)} - \\
    &\phantom{\mathbb{P}\Biggl(\max_{t_k < t < \eta^{k+2}} \sqrt{\frac{s_{t, k}+2}{2\alpha(1-\alpha)}}\Biggl\{}\ \sqrt{\frac{2\alpha(1-\alpha)\log \left(\frac{1}{h(s_{t, k})}\right)}{s_{t, k}+2}}\Biggr\} > \sqrt{\log\left(\frac{\sum_{t=t_k+1}^{\lfloor \eta^{k+2}\rfloor}h(s_{t, k})}{\gamma g(k)}\right)}\Biggr) \leq \gamma g(k)
\end{align*} for all $\gamma > 0$, $k \geq 0$. Since $\sqrt{a + b} \leq \sqrt{a} + \sqrt{b}$, it follows that \begin{align*}
    &\mathbb{P}\Biggl(\max_{t_k < t < \eta^{k+2}} \sqrt{\frac{s_{t, k}+2}{2\alpha(1-\alpha)}}\left\{\frac{\lceil (s_{t, k} + 1)(1-\alpha + u_{t, k})\rceil}{s_{t, k} + 1} - U_{\lceil (s_{t, k} + 1)(1-\alpha + u_{t, k})\rceil:s_{t, k}} - u_{t, k}\right\} + \frac{\sqrt{\pi}}{2} \\
    &\phantom{\mathbb{P}\Biggl(}\ > \sqrt{\log\left(\frac{1}{\gamma}\right)}\Biggr) \leq \gamma g(k)
\end{align*} for all $\gamma > 0$, $k \geq 0$. Applying a union bound argument again, we obtain that \begin{align*}
    &\mathbb{P}\Biggl(\max_{k \geq 0}\max_{t_k < t < \eta^{k+2}} \sqrt{\frac{s_{t, k}+2}{2\alpha(1-\alpha)}}\left\{\frac{\lceil (s_{t, k} + 1)(1-\alpha + u_{t, k})\rceil}{s_{t, k} + 1} - U_{\lceil (s_{t, k} + 1)(1-\alpha + u_{t, k})\rceil:s_{t, k}} - u_{t, k}\right\} + \frac{\sqrt{\pi}}{2} \\
    &\phantom{\mathbb{P}\Biggl(}\ > \sqrt{\log\left(\frac{1}{\gamma}\right)}\Biggr) \leq \gamma
\end{align*} for all $\gamma > 0$. Setting $\epsilon = \sqrt{\log\left(\frac{1}{\gamma}\right)}$ so that $\gamma = e^{-\epsilon^2}$, it follows that \begin{align*} 
    &\mathbb{P}\left(\max_{k \geq 0}\max_{t_k < t < \eta^{k+2}} \sqrt{\frac{s_{t, k}+2}{2\alpha(1-\alpha)}}\left\{\frac{\lceil (s_{t, k} + 1)(1-\alpha + u_{t, k})\rceil}{s_{t, k} + 1} - U_{\lceil (s_{t, k} + 1)(1-\alpha + u_{t, k})\rceil:s_{t, k}} -u_{t, k}\right\} + \frac{\sqrt{\pi}}{2} > \epsilon\right) \\
    &\quad\leq e^{-\epsilon^2}
\end{align*} for all $\epsilon > 0$. Thus, \begin{align*}
    \mathbb{E}\left[\max_{k \geq 0}\max_{t_k < t < \eta^{k+2}} \sqrt{\frac{s_{t, k}+2}{2\alpha(1-\alpha)}}\left\{\frac{\lceil (s_{t, k} + 1)(1-\alpha + u_{t, k})\rceil}{s_{t, k} + 1} - U_{\lceil (s_{t, k} + 1)(1-\alpha + u_{t, k})\rceil:s_{t, k}} -u_{t, k}\right\}\right] \leq 0
\end{align*} so \begin{align*}
    \mathbb{E}\left[\max_{k \geq 0}\max_{t_k < t < \eta^{k+2}} \left\{\frac{\lceil (s_{t, k} + 1)(1-\alpha + u_{t, k})\rceil}{s_{t, k} + 1} - U_{\lceil (s_{t, k} + 1)(1-\alpha + u_{t, k})\rceil:s_{t, k}} -u_{t, k}\right\}\right] \leq 0.
\end{align*} Finally, we obtain that \begin{align*}
    &\mathbb{E}\left[\min_{k \geq 0} \min_{t_k < t < \eta^{k+2}} U_{\lceil (s_{t, k} + 1)(1-\alpha + u_{t, k})\rceil:s_{t, k}}\right] \\
    &= \mathbb{E}\Biggl[-\max_{k \geq 0} \max_{t_k < t < \eta^{k+2}} \Biggl\{\frac{\lceil (s_{t, k} + 1)(1-\alpha + u_{t, k})\rceil}{s_{t, k} + 1} - U_{\lceil (s_{t, k} + 1)(1-\alpha + u_{t, k})\rceil:s_{t, k}} - \\
    &\phantom{=\mathbb{E}\Biggl[-\max_{k \geq 0} \max_{t_k < t < \eta^{k+2}} \Biggl\{}\ u_{t, k} - \frac{\lceil (s_{t, k} + 1)(1-\alpha + u_{t, k})\rceil}{s_{t, k} + 1} + u_{t, k}\Biggr] \\
    &\geq \min_{k \geq 0} \min_{t_k < t < \eta^{k+2}} \Biggl\{\frac{\lceil (s_{t, k} + 1)(1-\alpha + u_{t, k})\rceil}{s_{t, k} + 1} + u_{t, k}\Biggr\} \\
    &\geq 1-\alpha.
\end{align*}
\end{proof}

\subsection{Proof of Theorem \ref{dynamic_anytime_pac_thm}} \label{dynamic_anytime_pac_proof}
The result relies on the following lemma about uniform order statistics. 
\begin{lemma} \label{dynamic_anytime_pac_lemma}
Let $s_{t, k} := t - \left(\lceil \eta^k\rceil - 1\right)$ for some $\eta > 0$. Suppose that $U_1, \dots, U_{s_{t, k}}$ are independent standard uniform random variables and take $U_{1:s_{t, k}} \leq \dots \leq U_{s_{t, k}:s_{t, k}}$ to be the order statistics of $\{U_s\}_{s=1}^{s_{t, k}}$. Let $h, g : \mathbb{Z}^{\geq 0} \to \mathbb{R}$ satisfy $\sum_{t=0}^{\infty} h(t) = 1$. Then, for any $\alpha, \delta \in (0, 1)$, \begin{equation*}
    \mathbb{P}\left(\min_{k \geq 0} \min_{t_k < t < \eta^{k+2}} U_{\beta:s_{t, k}} \geq 1 - \alpha\right) \geq 1-\delta
\end{equation*} where $\beta \geq \left(s_{t, k} + 1\right)(1-\alpha)$ is the smallest natural number such that \begin{equation*}
    \psi\left(1-\alpha, \frac{\beta}{s_{t, k} + 1}\right) \geq u_{t, k} := (s_{t, k} + 1)^{-1}{\log\left(\sum_{t' = t_k + 1}^{\lfloor \eta^{k+2}\rfloor}\frac{ h(s_{t', k})}{\delta g(k)h(s_{t, k})}\right)}
\end{equation*} and $t_k > \eta^k$ is the smallest real number such that $\beta \leq s_{t, k}$ for all $t_k < t < \eta^{k+2}$.
\end{lemma}
Taking this lemma as given, the result follows from writing the desired guarantee in terms of uniform order statistics. As in the proof of Theorem \ref{dynamic_anytime_conformal_thm}, in order to satisfy \eqref{eq:anytime_pac_goal}, it suffices to choose $\beta$ such that \begin{equation*}
    \mathbb{P}\left(\min_{k \geq 0} \min_{t_k < t < \eta^{k+2}} U_{\beta:s_{t, k}} \geq 1 - \alpha\right) \geq 1 - \delta 
\end{equation*} where $U_{1:s_{t, k}}, \dots, U_{s_{t,k}:s_{t,k}}$ are uniform order statistics. From here, applying Lemma \ref{dynamic_anytime_pac_lemma} implies the desired result. 
\begin{proof}[Lemma \ref{dynamic_anytime_pac_lemma}]
We can follow an analogous proof to Lemma \ref{fixed_anytime_pac_lemma} to obtain for any $k \geq 0$, \begin{equation*}
    \mathbb{P}\left(\min_{t_k < t < \eta^{k+2}} U_{\beta:s_{t, k}} \geq 1 -\alpha\right) \geq 1 - \sum_{t=t_k+1}^{\lfloor \eta^{k+2}\rfloor} e^{-\left(s_{t,k} + 1\right)\psi\left(1-\alpha, \frac{\beta}{s_{t, k} + 1}\right)}
\end{equation*} if $\beta \geq \left(s_{t, k} + 1\right)(1-\alpha)$. Now, if we choose $\beta$ such that $\psi\left(1-\alpha, \frac{\beta}{s_{t, k} + 1}\right) \geq u_{t, k}$, it follows that \begin{align*}
    \mathbb{P}\left(\min_{t_k < t < \eta^{k+2}} U_{\beta:s_{t, k}} \geq 1 -\alpha\right) \geq 1 - \delta g(k).
\end{align*} Applying a union bound argument over $k$ yields \begin{align*}
    \mathbb{P}\left(\min_{k \geq 0} \min_{t_k < t < \eta^{k+2}} U_{\beta:s_{t, k}} \geq 1 -\alpha\right) \geq 1 - \delta\sum_{k=0}^{\infty}g(k) = 1-\delta.
\end{align*}
\end{proof}

\subsection{Proof of Theorem \ref{thm:optimal_width}} \label{optimal_width_proof}
First, note that if $k^* = k-1$ \begin{align*}
    L(\widehat{C}_{t, \alpha}) = L(\widehat{C}_{t, k-1, \alpha}) \leq L(C^*_{\alpha}) + L(\widehat{C}_{t, k-1, \alpha} \triangle C^*_{\alpha})
\end{align*} 
and otherwise, if $k^* = k$, it must be the case that \begin{align*}
    L(\widehat{C}_{t, \alpha}) = L(\widehat{C}_{t, k^*, \alpha}) = L(\widehat{C}_{t, k, \alpha}) \leq L(\widehat{C}_{t, k-1, \alpha}) \leq L(C^*_{\alpha}) + L(\widehat{C}_{t, k-1, \alpha} \triangle C^*_{\alpha}),
\end{align*} so we can focus on $L(\widehat{C}_{t, k-1, \alpha} \triangle C^*_{\alpha})$. Now, let us rewrite rewrite $\widehat{C}_{t, k-1, \alpha}$ as \begin{align*}
    \widehat{C}_{t, k-1, \alpha} &= \{z : R(z; \widehat{\xi}^{(k-1)}) \leq \widehat{q}_{t, k-1, \alpha}\} \\
    &= \{z : R(z; \xi^*) \leq  \widehat{q}_{t, k-1, \alpha} + R(z; \xi^*) - R(z; \widehat{\xi}^{(k-1)})\} \\
    &= \{z : R(z; \xi^*) \leq q^*_{\alpha} + \widehat{q}_{t, k-1, \alpha} - q^*_{\alpha} + R(z; \xi^*) - R(z; \widehat{\xi}^{(k-1)})\}.
\end{align*} 
By Assumption~\ref{eq:Lipschitz-score}, it follows that 
\begin{align*}
    \widehat{C}_{t, k-1, \alpha} \subseteq \{z : R(z; \xi^*) \leq q^*_{\alpha} + |\widehat{q}_{t, k-1, \alpha} - q^*_{\alpha}| + d(\xi^*,\,\widehat{\xi}^{(k-1)})\}
\end{align*} 
and 
\begin{align*}
    \widehat{C}_{t, k-1, \alpha} \supseteq\{z : R(z; \xi^*) \leq q^*_{\alpha} - |\widehat{q}_{t, k-1, \alpha} - q^*_{\alpha}| - d(\xi^*,\,\widehat{\xi}^{(k-1)})\}.
\end{align*}
Thus,  
\begin{align*}
    L(\widehat{C}_{t, k-1, \alpha} \triangle C^*_{\alpha})
    &\leq L(\{z: |R(z; \xi^*) - q^*_{\alpha}| < |\widehat{q}_{t, k-1, \alpha} - q^*_{\alpha}| + d(\xi^*,\,\widehat{\xi}^{(k-1)})\}) \\
    &= g(|\widehat{q}_{t, k-1, \alpha} - q^*_{\alpha}| + d(\xi^*,\,\widehat{\xi}^{(k-1)}))
\end{align*}
Let $F(r; \xi) = \mathbb{P}(R(Z;\xi) \le r|\xi)$, $F^*(r) := \mathbb{P}(R(Z; \xi^*) \le r)$, and
\begin{equation*}
    \widehat{F}_{t, k-1}(r; \widehat{\xi}^{(k-1)}) = \frac{1}{t - \lceil \eta^{k-1} \rceil + 2}\sum_{s= \lceil\eta^{k-1}\rceil}^t \mathbbm{1}\{R(Z_s; \widehat{\xi}^{(k-1)}) \leq r\}.
\end{equation*}
It follows from~\ref{eq:Lipschitz-score} that
\[
F^*\left(r - d(\xi^*,\,\widehat{\xi}^{(k-1)})\right) ~\le~ F(r; \widehat{\xi}^{(k-1)}) ~\le~ F^*\left(r + d(\xi^*,\,\widehat{\xi}^{(k-1)})\right). 
\]
Similarly, 
\[
\widehat{F}_{t,k-1}\left(r - d(\xi^*,\,\widehat{\xi}^{(k-1)})\right) ~\le~ \widehat{F}_{t,k-1}(r) ~\le~ \widehat{F}_{t,k-1}\left(r + d(\xi^*,\,\widehat{\xi}^{(k-1)})\right). 
\]
By the DKW inequality~\citep{massart1990}, 
\[
\mathbb{P}\left(\sup_{r\in\mathbb{R}}\left|\widehat{F}_{t,k-1}(r) - F(r; \widehat{\xi}^{(k-1)})\right| \ge \sqrt{\frac{\log(2/\kappa)}{2(t - \lceil\eta^{k-1}\rceil + 2)}}\,\middle|\,\widehat{\xi}^{(k-1)}\right) \le \kappa.
\]
Note that $t - \lceil\eta^{k-1}\rceil + 2 \ge \eta^{k-1}(\eta - 1) \ge t(\eta - 1)/\eta$. Therefore, with probability at least $1 - \kappa$, we have
\begin{equation}\label{eq:distribution-inequalities-DKW}
    F^*\left(r- d(\xi^*,\,\widehat{\xi}^{(k-1)})\right) - \sqrt{\frac{\eta\log(2/\kappa)}{2t(\eta - 1)}} \le \widehat{F}_{t,k-1}(r) \le F^*\left(r + d(\xi^*,\,\widehat{\xi}^{(k-1)})\right) + \sqrt{\frac{\eta\log(2/\kappa)}{2t(\eta - 1)}}.
\end{equation}
From Algorithm \ref{alg:dynamic_algo-illustration} and Theorem \ref{dynamic_anytime_conformal_thm}, we have
\begin{equation*}
    \widehat{q}_{t, k-1, \alpha} = \widehat{F}^{-1}_{t, k-1}(1 - \alpha + u_{t,k-1})
\end{equation*} 
and $h, g$ satisfy $\sum_{k=0}^{\infty}h(k) = \sum_{k=0}^{\infty}g(k) = 1$. 
Thus,~\eqref{eq:distribution-inequalities-DKW} implies that with probability $1 - \kappa$, 
\[
q^*_{\alpha + u_{t, k-1} - \sqrt{\eta\log(2/\kappa)/(2t(\eta - 1))}} - d(\xi^*,\,\widehat{\xi}^{(k-1)}) \le \widehat{q}_{t,k-1,\alpha} \le q^*_{\alpha + u_{t, k-1} + \sqrt{\eta\log(2/\kappa)/(2t(\eta - 1))}} + d(\xi^*,\,\widehat{\xi}^{(k-1)}).
\]
Summarizing, we get
\[
|\widehat{q}_{t,k-1,\alpha} - q^*_{\alpha}| \le d(\xi^*,\,\widehat{\xi}^{(k-1)}) + \max_{\omega\in\{\pm 1\}}\left|q^{*}_{\alpha} - q^*_{\alpha + u_{t, k-1} + \omega\sqrt{\eta\log(2/\kappa)/(2t(\eta - 1))}}\right|,
\]
and hence,
\[
L(\widehat{C}_{t,k-1,\alpha}\Delta C^*_{\alpha}) \le g\left(2d(\xi^*,\,\widehat{\xi}^{(k-1)}) + \max_{\omega\in\{\pm 1\}}\left|q^{*}_{\alpha} - q^*_{\alpha + u_{t, k-1} + \omega\sqrt{\eta\log(2/\kappa)/(2t(\eta - 1))}}\right|\right)
\] with probability at least $1-\kappa$. This proves the result. 
\end{document}